\documentclass{article}

 \usepackage[main, final]{neurips_2026}
\usepackage[utf8]{inputenc} % allow utf-8 input
\usepackage[T1]{fontenc}    % use 8-bit T1 fonts
\usepackage{hyperref}       % hyperlinks
\usepackage{url}            % simple URL typesetting
\usepackage{booktabs}       % professional-quality tables
\usepackage{amsfonts}       % blackboard math symbols
\usepackage{nicefrac}       % compact symbols for 1/2, etc.
\usepackage{microtype}      % microtypography
\usepackage{xcolor}         % colors
\usepackage{booktabs} 
\usepackage{graphicx}
\usepackage{subfigure}
\usepackage{caption}
\usepackage{amsmath}
\usepackage{amssymb}
\usepackage{algorithm}
\usepackage{algorithmic}
\usepackage{placeins}
\usepackage{wrapfig}
\usepackage{amsthm}
\usepackage{colortbl}
\usepackage{diagbox}
\usepackage{framed}
\usepackage{mdframed}
\usepackage{multirow}

\definecolor{shadecolor}{rgb}{0.92,0.92,0.92}

\newcommand{\bc}[1]{\mbox{\boldmath $\mathcal{#1}$}}

\newcommand{\mf}[1]{\mathbf{#1}}
\newcommand{\mb}[1]{\mathbb{#1}}

\newcommand{\T}{\mathrm{T}}

\usepackage[normalem]{ulem}

\definecolor{orange}{RGB}{255,107,0}

\newtheorem{theorem}{Theorem}

\newtheorem{remark}[theorem]{Remark}
\title{StreamPhy: Streaming Inference of High-Dimensional Physical Dynamics via State Space Models}

\author{%
	Panqi Chen\textsuperscript{$\mathrm{1}$} \quad Yifan Sun\textsuperscript{$\mathrm{1}$}   \quad Shikai Fang\textsuperscript{$\mathrm{1}$} \quad Xiao Fu\textsuperscript{$\mathrm{2}$} \quad Lei Cheng\textsuperscript{$\mathrm{1}$} \thanks{Correspond to lei{\_}cheng@zju.edu.cn} \\
	\textsuperscript{$\mathrm{1}$} College of Information Science and Electronic Engineering, Zhejiang University\\
	\textsuperscript{$\mathrm{2}$} Department of ECE, University of Minnesota - Twin Cities 	
}

\begin{document}

\maketitle

\begin{abstract}

%Inferring the evolution of physical fields from sparse  measurements in real time is a fundamental challenge in  science and engineering. The challenges often arise from the high dimensionality and multi-modality (e.g., spatio-temporal) nature of the fields and the presence of noisy, incomplete, and irregularly sampled observations. Although many latent diffusion-based methods enable continuous spatiotemporal modeling, they remain largely limited to offline settings and require full temporal observations during training and inference.  We propose \textbf{StreamPhy}, an end-to-end framework for streaming inference of full-field physical dynamics from irregular sparse observations. StreamPhy integrates a masking-based observation encoder with a structured state-space model, enabling efficient updates from irregular observations. At its core, we propose a Functional Tensor Feature-wise Linear Modulation (FT-FiLM) module, which provides an expressive conditional representation for continuous fields. We theoretically show that FT-FiLM  generalizes the functional Tucker model, offering enhanced expressiveness for complex dynamics. Extensive experiments across multiple physical systems demonstrate that StreamPhy achieves state-of-the-art performance in both accuracy and efficiency under diverse challenging observation settings.

Inferring the evolution of high-dimensional and multi-modal (e.g., spatio-temporal) physical fields from  irregular sparse measurements in real time is a fundamental challenge in science and engineering. Existing approaches, including diffusion-based generative models and functional tensor methods, typically operate in offline settings, depend on full temporal observations, or incur substantial inference cost. We propose StreamPhy, an end-to-end framework that enables efficient and accurate streaming inference of full-field physical dynamics from incoming irregular sparse measurements. To realize this, StreamPhy introduces a data-adaptive observation encoder that is robust to arbitrary observation patterns, leverages a structured state-space model for memory-efficient online updates across irregular time intervals, and develops an expressive Functional Tensor Feature-wise Linear Modulation (FT-FiLM) decoder for continuous-field generation. We further establish theoretically that FT-FiLM is more expressive than the functional Tucker model, admitting a richer function class for representing complex physical dynamics.  Experiments on three representative physical systems under challenging sampling patterns show that StreamPhy consistently outperforms state-of-the-art baselines, with at least 48\% improvement in accuracy and up to 20--100$\times$ faster inference than diffusion-based methods. 

\end{abstract}

\section{Introduction}
The real-time inference of spatiotemporal physical fields from streaming observations is of fundamental importance across a wide range of scientific and engineering applications. In applications such as environmental sensing~\cite{bttd}, fluid dynamics~\cite{pnas}, and structural health monitoring~\cite{structural}, observations are often sparse, irregular, and acquired under time-varying measurement patterns. These characteristics pose significant challenges for building models that can efficiently assimilate incoming data, maintain temporal consistency, and produce accurate full-field reconstructions.

% {\red Recent advances in generative models, including diffusion models~\cite{huang2024diffusionpde, du2024conditional, li2024learning_nmi2, sdift2025} and flow-matching models~\cite{functionalFLowMatching, physense}, have shown strong performance in physics-based tasks such as PDE solving, super-resolution, field reconstruction, and sensor placement optimization. A common paradigm in generative physical modeling follows two stages: first, an encoder is learned to map observations from the data space to a latent space, with vision transformers~\cite{fundiff} (ViT) and more recent joint embedding predictive architectures~\cite{JEPA} (JEPAs) serving as representative examples; second, a generative model is learned in the latent space to capture the distribution of latent variables. During inference, diffusion models further exploit observation likelihoods through diffusion posterior sampling (DPS)~\cite{chung2023diffusion} and then map the generated latent variables back to the observation space. } 
% {\bf {\red how is this part realated to streaming physical fields and spatiotemporal?}}

Recent years have witnessed growing recognition that modern artificial intelligence (AI) techniques hold strong promise for accurately modeling spatiotemporal physical fields. One line of work~\cite{huang2024diffusionpde, du2024conditional, li2024learning_nmi2, fundiff} draws inspiration from two-stage diffusion paradigms. Specifically, these methods first encode observations into a latent space using  architectures such as vision transformers~\cite{fundiff} or joint embedding predictive architectures~\cite{JEPA}, and then learn a generative model over the latent variables. During inference, techniques such as diffusion posterior sampling (DPS)~\cite{chung2023diffusion} are employed to incorporate observations. Despite their strong empirical performance, these approaches typically rely on vectorized or matrix-form representations, making them readily compatible with off-the-shelf AI models (e.g., CNNs~\cite{salimans2017pixelcnn} and Transformers~\cite{attention}) but less effective at explicitly capturing the intrinsic spatiotemporal structure of physical fields. Another line of work adopts a tensor decomposition perspective, often combined with functional representation learners such as implicit neural representations (INRs) to accommodate continuous domains~\cite{OFTD2025, catte2025, luo2023lowrank, fang2023functional}. By explicitly modeling the multilinear structure of spatiotemporal data, tensor-based methods are better suited to preserve cross-dimensional relationships inherent in physical systems. Nonetheless, many existing approaches are primarily designed for batch settings rather than streaming scenarios, e.g.,~\cite{luo2023lowrank, catte2025}. A notable exception is~\cite{OFTD2025}, which supports incremental updates with  arriving observations, rendering it naturally well-suited for streaming spatiotemporal reconstruction.

\begin{figure}[t]
	\centering
	\includegraphics[width=1\linewidth]{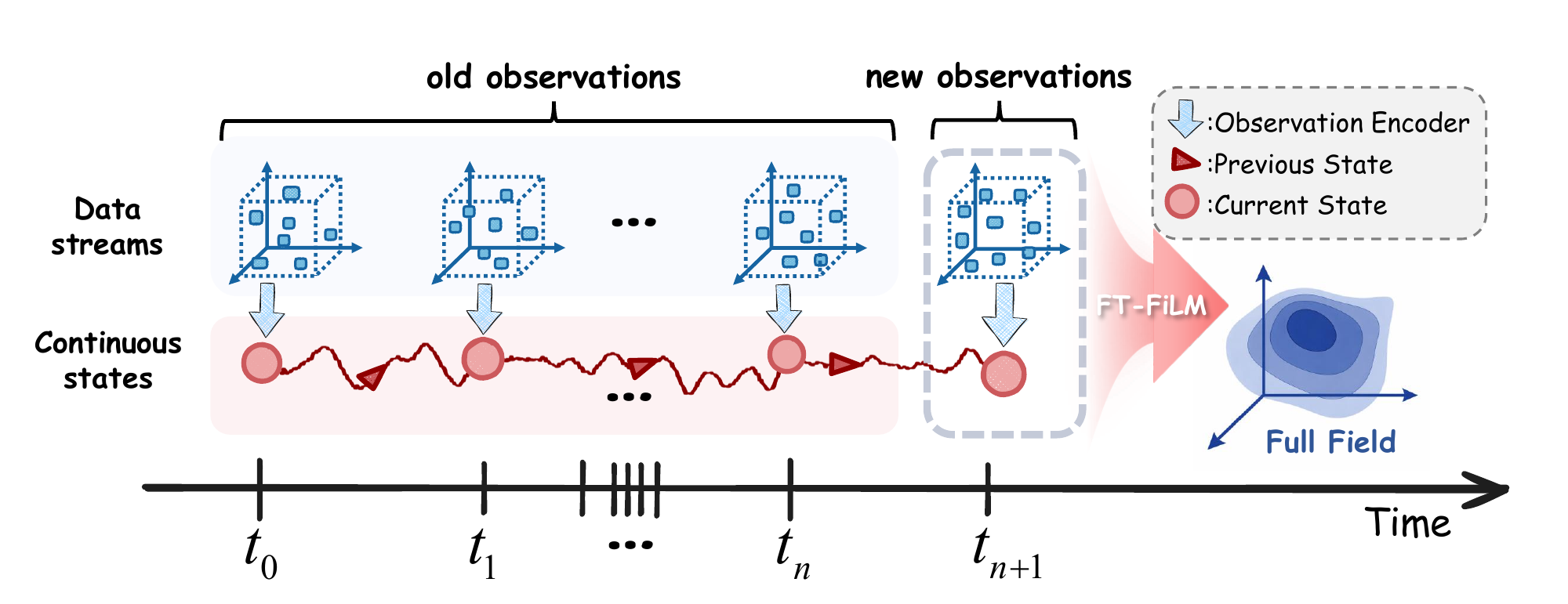}
	\caption{Semantic illustration of the proposed StreamPhy framework.}
	\label{fig:flowchart}
	\vspace{-0.32in}
\end{figure}

% In parallel, functional tensor decomposition methods~\cite{OFTD2025, catte2025, luo2023lowrank, fang2023functional} offer a complementary viewpoint by explicitly exploiting the multilinear structure inherent in continuous spatiotemporal data. In particular, \cite{OFTD2025} supports incremental updates with newly arriving observations, rendering it naturally well-suited for streaming scenarios.

 These methods all provide viable options for physics field estimation to a good extent, but some limitations remain. For example, diffusion-based models~\cite{sdift2025, du2024conditional, li2024learning_nmi2} are limited to offline settings and require full temporal observations during training and inference. In addition, their two-stage nature could increase the risk of error propagation. The functional tensor approaches~\cite{luo2023lowrank, OFTD2025, catte2025} rely on implicit neural representation (INR)~\cite{sitzmann2020implicit} and thus only fit to given data,lacking the ability to generalize to future data effectively.

{ \noindent
{\bf Contributions.}
To enable streaming inference of continuous spatiotemporal dynamics from sparse and irregular observations, we propose  \textbf{StreamPhy}, an end-to-end framework built upon HiPPO-based state-space model (SSM)~\cite{hippo,s4,lssl}. SSM offer a principled and efficient mechanism for sequential modeling with online updates and long-range dependency capture, while recent advances such as Mamba~\cite{mamba} have demonstrated their strong long-context modeling capability. However, applying SSMs to physical data is nontrivial, as scientific observations are often irregularly sampled in both space and time, and reconstructing continuous spatial fields from latent states demands expressive functional decoding. To address these challenges, StreamPhy integrates three key components: a {\it data-adaptive observation encoder} that transforms arbitrarily sampled measurements into structured latent representations, a {\it HiPPO-based SSM} for memory-efficient temporal evolution modeling, and a novel {\it Functional Tensor Feature-wise Linear Modulation (FT-FiLM) decoder} for reconstructing continuous spatial fields. The proposed FT-FiLM offers  greater expressive power than existing functional tensor models through flexible feature-wise modulation, enabling more effective modeling of complex spatiotemporal dynamics. A semantic illustration of StreamPhy is provided in Fig.~\ref{fig:flowchart}.}

% {\blue
% \noindent
% {\bf Contributions.} To address these challenges, we propose StreamPhy, an end-to-end framework for streaming inference of continuous spatiotemporal dynamics from sparse and irregular observations. StreamPhy integrates three key components: a data-adaptive observation encoder, a structured state-space model (SSM), and an expressive functional decoder. 

% First, we develop neural observation encoders that transform arbitrarily sampled measurements into structured latent representations, enabling robust feature extraction under heterogeneous observation patterns and seamless compatibility with sequence modeling. 

% Second, the encoded representations are processed by an SSM to capture temporal evolution and long-range dependencies in a memory-efficient streaming manner, without storing historical data. 

% Third, we introduce a novel Functional Tensor Feature-wise Linear Modulation (FT-FiLM) module for decoding latent states into continuous spatial fields. FT-FiLM generalizes existing functional tensor models through flexible functional modulation, offering enhanced expressive power for modeling complex spatiotemporal dynamics. 

% %By integrating these components into a unified framework, StreamPhy overcomes key limitations of both generative and tensor-based approaches while enabling efficient and scalable stream inference. 
% A semantic illustration of StreamPhy is provided in Fig.~\ref{fig:flowchart}.
% }

We evaluate the proposed method on three representative physical systems with diverse sampling patterns. Experimental results demonstrate that StreamPhy consistently outperforms diffusion-based methods and online tensor approaches by a significant margin in both accuracy and efficiency, achieving at least a 48\% improvement in VRMSE over state-of-the-art methods and up to $20\sim 100\times$ faster inference compared to diffusion-based models. These results highlight the effectiveness of StreamPhy for real-time streaming inference.
%{\red we could itemize the above if you want. But introduction could stop here already}

%Our contributions are summarized as follows: {\red the below may be repetitive. }
%\begin{itemize}
%\item We propose StreamPhy, a unified framework for streaming inference of high-dimensional spatiotemporal physical fields from sparse and irregular observations, where irregularity refers to arbitrarily distributed sampling locations and time-varying observation patterns.

%\item We design data-friendly observation encoders that map heterogeneous and sparsely sampled inputs into structured representations compatible with sequence modeling, improving robustness under diverse sampling schemes.

%\item We introduce the FT-FiLM module for expressive functional decoding and theoretically show that it generalizes the functional Tucker model by subsuming its function class, while offering greater flexibility for modeling complex dynamics.

%\item We seamlessly integrate the above components into a structured state-space model, enabling efficient and memory-free online inference, and demonstrate state-of-the-art performance across multiple physical systems.
%\end{itemize}

{\bf Notation:} Lower- and upper-case bold letters (e.g., $\mf x$ and $\mf X$) denote vectors and matrices, respectively. Upper-case bold calligraphic letters and upper-case calligraphic letters (e.g., $\bc{X}$ and $\mathcal{X}$) denote tensors and sets, respectively.

\section{Background}
\subsection{Preliminary}
\textbf{Tensor Decomposition:} Consider a $K$-mode tensor $\bc{Y} \in \mathbb{R}^{I_1 \times \cdots \times I_K}$, where each entry $y_{\mf{i}}$ is indexed by $\mf{i}=(i_1,\dots,i_K)$ { where $i_k\in[I_K]$ for all $k$}. The Tucker decomposition~\cite{sidiropoulos2017tensor} factorizes $\bc{Y}$ into mode-wise factor matrices $\{\mf{U}^{k} \in \mathbb{R}^{I_k \times R_k}\}_{k=1}^{K}$ and a core tensor $\bc{G} \in \mathbb{R}^{R_1 \times \cdots \times R_K}$, where $R_k$ denotes the rank of mode $k$. Each entry admits the multilinear form:$	y_{\mf{i}} \approx \text{vec}(\bc{G})^{\T}(\mf{u}_{i_1}^1\otimes \cdots \otimes \mf{u}_{i_K}^K)
= \sum_{r_1=1}^{R_1}\cdots \sum_{r_K=1}^{R_K} w_{r_1,\dots,r_K} \prod_{k=1}^{K} u^k_{i_k,r_k},$ where $\mf{u}_{i_k}^k \in \mathbb{R}^{R_k}$ is the latent representation of index $i_k$ in mode $k$. The CANDECOMP/PARAFAC (CP) decomposition~\cite{HarshmanCP} arises as a special case with equal ranks $R_k$ and a superdiagonal core tensor $\bc{G}$. { Conventional tensors are defined over discrete grids, indexed by ${\bf i}$.}
Functional tensors extend this framework to continuous domains by modeling entries as evaluations of a multivariate function over $\mathbb{R}^K$, rather than discrete grids. Applying the Tucker structure yields the functional Tucker decomposition~\cite{luo2023lowrank,fang2023functional}:
\begin{equation}
	\begin{split}
		y(\mf{i}) \approx
        \text{vec}(\bc{G})^{\T}
		\bigl(\mf{u}^1(i_1)\otimes \cdots \otimes \mf{u}^K(i_K)\bigr),
		\label{eq:func_tucker}
	\end{split}
\end{equation}
where $\mf{i}=(i_1,\dots,i_K)\in \mathbb{R}^{K}$  ({ where each $i_k$ takes continuous values}), and $\mf{u}^k(\cdot): \mathbb{R} \to \mathbb{R}^{R_k}$ denotes the latent function for mode $k$.  Such formulations have proven effective for modeling high-dimensional data over continuous coordinates, including climate, turbulent flow, and geospatial datasets~\cite{luo2023lowrank, fang2023functional, catte2025, OFTD2025}.

\textbf{A  state space model}~\cite{lssl,s4,hippo} (SSM) provides a general mathematical framework for modeling recurrent processes via the evolution of a latent state. We note that term ``state space model'' has been used broadly in prior work~\cite{kalman, puterman1990markov, eddy1996hidden}; here, we use ``SSM'' exclusively to denote a class of  structured SSMs.

 In particular, we consider a continuous-time dynamical system that maps a one-dimensional input \( s(t) \in \mathbb{R} \) to an output \( \eta(t) \in \mathbb{R} \) through a $L$-dimensional latent state \( \mf{x}(t) \in \mathbb{R}^L \), formulated as
\begin{equation}
	\begin{split}
		\dot{\mf x}(t) &= \mathbf{A} \mf x (t) + \mathbf{b} s(t), \\
		\eta(t) &= \mathbf{c}^{\T} \mf x(t) + d \cdot s(t),
	\end{split}
	\label{eq:func_ssm}
\end{equation}
where \(\mathbf{A} \in \mb{R}^{L \times L}\) governs the state transition dynamics, and \(\mathbf{b} \in \mb{R}^{L}, \mathbf{c} \in \mb{R}^{L}, d \in \mb{R}\) parameterize the input-to-state and state-to-output mappings.
This formulation  captures temporal dependencies through latent state evolution in a continuous-time dynamical system, which can be used as a black-box representation in a deep learning model. 

Prior works~\cite{hippo,lssl,s4} leverage HiPPO theory to enable online compression of continuous signals and discrete-time sequences via projections onto orthogonal polynomial bases, thereby significantly enhancing the modeling of long-range dependencies in Eq.~\ref{eq:func_ssm}. Specifically, this framework prescribes a structured class of matrices \(\mathbf{A}\) and \(\mathbf{b}\), allowing the latent state \(\mathbf{x}(t)\) to effectively retain the history of the input \(s(t)\) in an online manner. A prominent instance, HiPPO-LegS, is defined as:

\begin{equation}
	\mf{A}_{lk} =
	-\begin{cases}
		\sqrt{(2l+1)(2k+1)} & \text{if } l > k, \\
		l+1 & \text{if } l = k, \\
		0 & \text{if } l < k.
	\end{cases}; \quad  \mf{b}_{l} = \sqrt{2l+1}.
	\label{eq:hippo}
\end{equation}

To accommodate discrete input sequences, the continuous-time SSM can be discretized into a recurrence form using the bilinear transform~\cite{btmethod}. Given a step size $\Delta t$, the resulting discrete-time system is
\begin{align}
 	\mf x_{t} &= \bar{\mathbf{A}} \mf x_{t-1} + \bar{\mathbf{b}} s_t, \label{eq:ssm1} \\
	\eta_{t} &= \mathbf{c}^{\T} \mf x_t + d \cdot s_t, \label{eq:ssm2}
\end{align}
where 
\(\bar{\mathbf{A}} = \left(\mathbf{I} - \frac{\Delta t}{2}\mathbf{A}\right)^{-1}
\left(\mathbf{I} + \frac{\Delta t}{2}\mathbf{A}\right)\), 
\(\bar{\mathbf{b}} = \left(\mathbf{I} - \frac{\Delta t}{2}\mathbf{A}\right)^{-1}\Delta t \mathbf{b}\) and $\mf{I}$ is the identity matrix. This discretization naturally enables the model to handle irregularly sampled time intervals.

\subsection{Problem Setting and Existing Methods}

\textbf{Problem Statement:} 
We denote the $K$-mode physical full-field tensor at time step $t$ by $\bc{Y}_t$, where $\mathbf{i}=(i_1,\cdots,i_K)$ denotes a  continuous spatial coordinate. 
Without loss of generality, we consider a sequence of $M$ observation time steps $\mathcal{T}=\{t_1,\dots,t_M\}$ with possibly nonuniform intervals, reflecting realistic asynchronous sampling. 
At each time $t_m$, we observe a set
$\mathcal{O}_{t_m}=\{(\mathbf{i}_{n_m}, y_{t_m,\mathbf{i}_{n_m}})\}_{n_m=1}^{N_m}$,
where $y_{t_m,\mathbf{i}_{n_m}}$ is the entry of $\bc{Y}_{t_m}$ at location $\mathbf{i}_{n_m}$ and $N_m$ may vary over time. 
The observation pattern is therefore time-varying, and we collect the entire data stream as $\mathcal{O}=\{\mathcal{O}_{t_m}\}_{m=1}^M$. 
Unlike prior settings,  e.g., those in~\cite{sdift2025, catte2025},  our streaming setting assumes that observations $\mathcal{O}_{t_m}$ arrive sequentially, requiring online updates based solely on current inputs. The objective is to perform streaming inference of the underlying continuous field $y_t(\mathbf{i}): \mathbb{R}^K \rightarrow \mathbb{R}$ at arbitrary spatial coordinates from this irregular, time-ordered observation stream, which can be viewed as an online function approximation problem.

\textbf{Existing Methods:}
Latent diffusion-based methods~\cite{du2024conditional, sdift2025, li2024learning_nmi2} encode observations into a latent space and learn a generative model, where reconstruction is guided by posterior sampling. However, such a paradigm typically relies on offline training and iterative sampling at inference, limiting its applicability in streaming settings. Within this framework, the functional Tucker model (FTM, see Eq.~\eqref{eq:func_tucker})  serves as an effective latent encoder for high-dimensional data \cite{sdift2025}. Nevertheless, FTM relies on multilinear interactions between core tensors and basis functions, often requiring large tensor ranks to achieve sufficient expressivity, which leads to increased memory overhead and slower latent-space operations.

In addition, existing functional tensor-based methods~\cite{luo2023lowrank, OFTD2025, catte2025, fang2023functional} leverage the multilinear structures of FTM for spatiotemporal reconstruction. Despite their effectiveness, these approaches are typically trained in an unsupervised manner on a single evolving tensor sequence, which limits their performance under extreme observations. Moreover, they often rely on high tensor ranks to capture complex dynamics and require memory replay for online updates, thereby reducing practicability.  Further discussion of the related work is provided in Appx.~\ref{app:related_work}.

\subsection{SSM-driven Framework: Motivation and Challenges}
Existing paradigms are fundamentally constrained by limited computational efficiency and insufficient long-range memory, thereby limiting their ability to effectively capture long-horizon temporal dependencies. This shortcoming motivates exploring  SSMs~\cite{lssl, hippo, s4}, which offer a principled and memory-efficient framework for sequential representation learning. In particular, the HiPPO framework~\cite{hippo} provides a rigorous mechanism for projecting historical information onto a set of orthogonal basis functions, enabling compact yet memory-preserving latent state representations. Such properties make SSMs especially well-suited for the online approximation of continuous-time dynamical systems.

Despite these advantages, online function approximation from sparse and sequential observations remains inherently challenging. The difficulty stems from the curse of dimensionality, the streaming and partially observed nature of the data, and the intrinsic complexity of spatiotemporal physical processes. Extending conventional SSM formulations to irregular spatiotemporal settings, which are characterized by non-uniform temporal sampling and incomplete spatial coverage, introduces additional challenges. Specifically, two critical issues must be addressed: (1) how to effectively encode irregular observations to consistently update latent states defined over a regular domain; and (2) how to construct a sufficiently expressive representer capable of accurately reconstructing full-field outputs from compact latent states.

To tackle these challenges, we propose \textbf{StreamPhy}, a deep state space framework that seamlessly integrates a tailored attention-based observation encoder, a HiPPO-based structured SSM, and a functional tensor-based feature-wise linear modulation module. Through this synergistic design, the proposed framework enables accurate and scalable modeling of high-dimensional continuous physical dynamics from sparse and irregular observations. 
%A schematic overview of the framework is presented in Fig.~\ref{fig:flowchart}.
\section{Methodology}
\subsection{Observation Encoder for Arbitrary Patterns}
As in Eq.~(\ref{eq:ssm1}), prior SSMs~\cite{hippo,lssl,s4} mainly target multivariate time series with regularly sampled observations, used either directly as $s_t$ or encoded via RNNs or CNNs. In contrast, at each time step $t_m$, { our} observation is a set $\mathcal{O}_{t_m}$ with varying cardinality, which cannot be directly handled by existing SSM models. To address this limitation, we aim to develop a novel observation encoder that maps each observation set into a structured latent representation compatible with the SSM framework.  Motivated by the ability of attention to naturally handle variable-length inputs, we design an attention-based observation encoder to effectively aggregate irregular observations sampled from arbitrary patterns, as illustrated in Fig.~\ref{fig:ob_en}.
 \begin{wrapfigure}{r}{0.31\textwidth}     
	\centering
	\includegraphics[width=\linewidth]{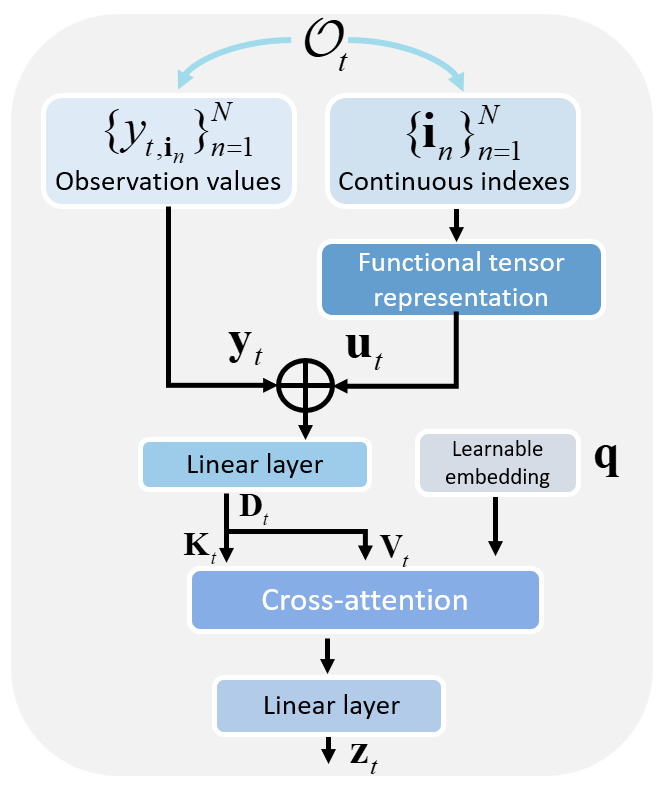}
	\caption{Details of the proposed single-head observation encoder module.}
	\label{fig:ob_en}
    \vspace{-4mm}
\end{wrapfigure}

Consider a time step $t$ (we omit the subscript $m$ for simplicity), where the observation is given by $\mathcal{O}_t$. We first decompose $\mathcal{O}_t$ into an observation value set $\{y_{t,\mathbf{i}_n}\}_{n=1}^{N}$ and a continuous index set $\{\mathbf{i}_n\}_{n=1}^{N}$. Each index $\mathbf{i}_n$ is then mapped to a functional tensor representation $\mf{u}_{n} \in \mathbb{R}^{\sum_{k=1}^{K}R_k}$ defined as
\begin{equation}
	\mf{u}_{n} = \text{Concat}\big[\mf{u}^{1}_{\boldsymbol{\theta}_1}(i_1),\cdots,\mf{u}^{k}_{\boldsymbol{\theta}_k}(i_k),\cdots, \mf{u}^{K}_{\boldsymbol{\theta}_K}(i_K)\big],
	\label{eq:pe}
\end{equation}
where $\mf{u}^{k}_{\boldsymbol{\theta}_k}(i_k): \mathbb{R}_{+} \to \mathbb{R}^{R_k}$ is a learnable function parameterized by $\boldsymbol{\theta}_k$ that maps the continuous coordinate $i_k$ to its latent embedding. Specifically, we adopt the idea of INR~\cite{sitzmann2020implicit}: $\mf{u}^{k}_{\boldsymbol{\theta}_k}(i_k) = \text{MLP}\big([\cos(2\pi\boldsymbol{\phi}_k i_k), \sin(2\pi\boldsymbol{\phi}_k i_k)]\big)$, where $\boldsymbol{\phi}_k$ is a learnable frequency vector that modulates the input coordinate.
Collecting all functional tensor representations and their corresponding observations, we obtain $\mf{y}_t \in \mathbb{R}^{N \times 1}$ and $\mf{U}_t \in \mathbb{R}^{N \times KR}$. We concatenate them and project the result into a latent representation $\mf{D}_t \in \mathbb{R}^{N \times D}$ via a linear transformation.

Next, we aim to aggregate the $N$ latent samples into a structured representation that is independent of $N$. A straightforward approach is to compute a weighted average over $\mf{D}_t$ along the sample dimension $N$. However, such a strategy treats all observation points uniformly and thus fails to capture informative variations across samples. To address this limitation, we introduce a learnable query embedding $\mf{q} \in \mathbb{R}^{D}$ and employ a cross-attention mechanism over $\mf{D}_t$ to adaptively select more informative observations~\cite{attention}:
\begin{equation}
\text{CrossAttention}(\mf q,\mf K_t,\mf V_t) = \text{Softmax}\left(\frac{\mf q^{\T} \mf K_t^{\T}}{\sqrt{D}}\right)\mf{V}_t \in \mathbb{R}^{1 \times D},
\end{equation}
where  	$\mf{K}_t = \mf{D}_t\mf{W}_{K} \in \mathbb{R}^{N\times D},  \mf{V}_t = \mf{D}_t\mf{W}_{V} \in \mathbb{R}^{N\times D}$ and    $\mf{W}_{K}, \mf{W}_{V} \in \mathbb{R}^{D \times D}$ are learnable projection matrices.
Finally, the aggregated representation is passed through an MLP to produce the output of the observation encoder:
\begin{equation}
	\mf{z}_t = \text{MLP}\big(\text{CrossAttention}(\mf q,\mf K_t,\mf V_t)\big) \in \mathbb{R}^{P}.
	\label{eq:cross_att}
\end{equation}
Note that  $\mf{q}$ is shared across all time steps, enabling the encoder to produce temporally consistent and structurally stable representations.

In { practice}, observations are typically sampled in non-uniform patterns, where both the number and spatial distribution of samples vary over time. To mimic this setting, we adopt a stochastic masking strategy that randomly removes a subset of observations during training, thereby encouraging the model to learn robust  representations.

We further extend the above formulation to a multi-head version to capture diverse dependencies with $H$ heads. For each head $h$, we compute $\mf{q}^{(h)} = \mf{q}[(h-1)d_h:hd_h], \quad
\mf{K}_t^{(h)} = \mf{D}_t\mf{W}_K^{(h)}, \quad
\mf{V}_t^{(h)} = \mf{D}_t\mf{W}_V^{(h)}$ and have $\text{head}_h = \text{Softmax}\left( \frac{ (\mf{q}^{(h)})^{\T} (\mf{K}_t^{(h)})^{\T}}{\sqrt{d_h}}\right)\mf{V}_t^{(h)}$, 
where $d_h = D/H$. The outputs of all heads are concatenated and projected:$	\text{MultiHead}(\mf{q}, \mf{K}_t, \mf{V}_t)
= \text{Concat}(\text{head}_1,\dots,\text{head}_H)\mf{W}_O.$ Finally, we have $\mf{z}_t = \text{MLP}\big(\text{MultiHead}(\mf{q}, \mf{K}_t, \mf{V}_t)\big)$.

With the proposed observation encoder, we extend the scalar input $s_t$ in Eq.~\eqref{eq:ssm2} to a $P$-dimensional representation $\mf{z}_t \in \mathbb{R}^{P}$, yielding the following new transition function:
\begin{equation}
	\mf{X}_{t} = \bar{\mathbf{A}} \mf{X}_{t-1} + \bar{\mathbf{b}} \circ \mf{z}_t,
	\label{eq:latent_state}
\end{equation}
where $\mf{X}_{t} \in \mathbb{R}^{L \times P}$ denotes the augmented state that tracks the evolution of $\mf{z}_t$ and $\circ$ denotes the outer product, producing a $L \times P$  injection term.

\subsection{Functional Tensor Feature-wise Linear Modulation as an Expressive Representer}
\begin{figure*}[t]
	\centering
	\begin{minipage}{0.5\textwidth}
		\centering
		\includegraphics[width=0.85\textwidth]{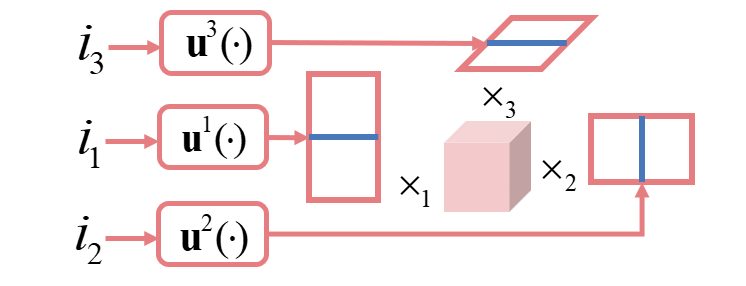}
		\caption*{(a) Functional Tucker model~\cite{luo2023lowrank, sdift2025}.  }
	\end{minipage}   
	\hspace{-0.3cm} % 调整两张图片之间的间距
	\begin{minipage}{0.5\textwidth}
		\centering
		\includegraphics[width=0.8\textwidth]{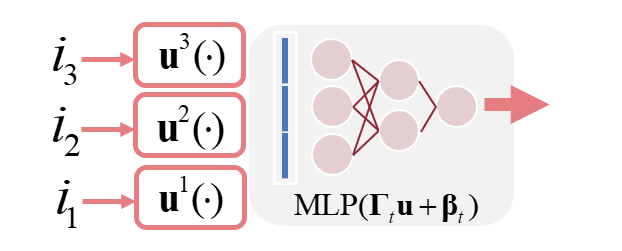}
		\caption*{(b) Functional tensor FiLM. }
	\end{minipage}    
	\caption{Illustrative comparison of FTM and FT-FiLM ($K$=3).}
	\label{fig:FTM-FTFILM}
	\vspace{-0.07in}
\end{figure*}

As shown in Eq.~(\ref{eq:ssm2}), existing SSMs typically decode signals via a linear combination of latent states and observations. However, such formulations suffer from limited expressive power and are not well-suited for high-dimensional full-field generation.

When performing reconstruction, we aim to enable querying the field value at any coordinate in the continuous domain. Formally, we seek a reconstruction function of the form
\begin{equation}
    y_t(\mathbf{i}_n) = f(\mathbf{X}_t, \mathbf{z}_t \mid \mathbf{u}_n),
\end{equation}
where $\mathbf{u}_n$ denotes the functional representation associated with coordinate $\mathbf{i}_n$, shown in Eq.~\eqref{eq:pe}. Such a formulation decouples spatial representation from temporal dynamics, allowing flexible conditioning on both the latent state and spatial features. This is particularly desirable for high-dimensional field generation, where expressive interactions between $\mathbf{X}_t$, $\mathbf{z}_t$, and $\mathbf{u}_n$ are essential.

{ We propose a Functional Tensor FiLM (FT-FiLM) module. FT-FiLM is an extension of Feature-wise Linear Modulation (FiLM)~\cite{film} into the functional tensor domain. FiLM is known for its effectiveness in conditional reasoning and generative modeling}. The key idea is to incorporate feature-wise affine modulation into a functional tensor representation, where both the scale and shift parameters are explicitly conditioned on the latent state, enabling flexible interaction between dynamics and spatial representations. Specifically, we first generate modulation parameters conditioned on both latent state $\mathbf{X}_t$ and observation latent $\mathbf{z}_t$:
\begin{equation}
    \boldsymbol{\Gamma}_t = f_{\boldsymbol{\omega}_1}(\mathbf{X}_t, \mathbf{z}_t), 
    \quad 
    \boldsymbol{\beta}_t = f_{\boldsymbol{\omega}_2}(\mathbf{X}_t, \mathbf{z}_t),
    \label{eq:gamma_beta}
\end{equation}
where $f_{\boldsymbol{\omega}_1}$ and $f_{\boldsymbol{\omega}_2}$ are learnable mappings. The detailed computation of these functions is provided in Appx.~\ref{app:detail}. Here, $\boldsymbol{\Gamma}_t \in \mb{R}^{ V \times \sum_{k=1}^{K}R_k }$ and $\boldsymbol{\beta}_t \in \mb{R}^{V}$ denote the feature-wise scaling and shifting parameters, respectively, which are used to modulate the latent representations in a feature-wise manner.

Given  $\mathbf{u}_n$ indexed at $\mathbf{i}_n$, the field value is then computed as 
\begin{equation}
    y_t(\mathbf{i}_n) = f_{\boldsymbol{\varphi}}\!\left( \boldsymbol{\Gamma}_t \mathbf{u}_n + \boldsymbol{\beta}_t \right),
    \label{eq:read_out}
\end{equation}
where $f_{\boldsymbol{\varphi}}: \mathbb{R}^V \rightarrow \mathbb{R}$ is a learnable readout function.

This formulation enables a more flexible and expressive coupling between latent dynamics and high-dimensional field representations, yielding improved capacity over the functional Tucker model (FTM)~\cite{luo2023lowrank, sdift2025}. Illustrative comparisons for $K=3$ are provided in Fig.~\ref{fig:FTM-FTFILM}. 
From a theoretical perspective, Theorem~\ref{thm:ftfilm} establishes that FT-FiLM is  more expressive than FTM in the sense of uniform approximation, with its closure coinciding with the entire space \(C(\Omega)\) (   {see the definition of closure in the theorem below and Appx.~\ref{app:sec:proof} for the proof and further explanations of the topological terminology}). This stronger expressivity provides a more flexible basis for modeling complex dynamics.

\begin{theorem}[Expressivity of FT-FiLM]
\label{thm:ftfilm}
Let $K \geq 2$, and let $\Omega_k \subset \mathbb{R}$ be compact intervals
with nonempty interior. Set $\Omega = \Omega_1 \times \cdots \times \Omega_K$.
Fix ranks $\mathbf{R} = (R_1, \ldots, R_K)$ with $R_k \geq 1$, and let
$S = \sum_{k=1}^{K} R_k$. Assume the FT-FiLM width satisfies $V \geq K$.
Define
\begin{align*}
&\mathcal{F}_{\mathrm{FTM}}(\mathbf{R})
= \Big\{\,\mathbf{i}\mapsto
\text{vec}(\bc{G})^{\T}\bigl(\mf{u}^1(i_1)\otimes \cdots \otimes \mf{u}^K(i_K)\bigr)
\;\Big|\;\boldsymbol{\mathcal{G}}\in\mathbb{R}^{R_1\times\cdots\times R_K},\;
\mathbf{u}^{k}\in C(\Omega_k,\mathbb{R}^{R_k})\,\Big\},\\[4pt]
&\mathcal{F}_{\mathrm{FT\text{-}FiLM}}(\mathbf{R}, V)
= \Big\{\,\mathbf{i}\mapsto f_{\boldsymbol{\varphi}}\!\big(
\mathbf{W}_g\, \mathbf{u}_{\mathbf{i}} + \mathbf{w}_b\big)
\;\Big|\;
\mathbf{u}_{\mathbf{i}}=\big[\mathbf{u}^{1}(i_1);\ldots;\mathbf{u}^{K}(i_K)\big]\in\mathbb{R}^S,\\
&\qquad\quad \qquad\quad \qquad\quad \quad \qquad\quad \qquad
\mathbf{W}_g\in\mathbb{R}^{V\times S},\;
\mathbf{w}_b\in\mathbb{R}^V,\;
\mathbf{u}^{k}\in C(\Omega_k,\mathbb{R}^{R_k}),\;
f_{\boldsymbol{\varphi}}\in\mathcal{N}\,\Big\},
\end{align*}
{   where $\mathcal{N}$ is a class of continuous feedforward neural networks
$f_{\boldsymbol{\varphi}}:\mathbb{R}^V\to\mathbb{R}$ satisfying the universal
approximation property~\cite{UAT1}.}
Let $\overline{(\,\cdot\,)}$ denote closure
under the uniform norm on $\Omega$.
{   For a function class $\mathcal{F}\subseteq C(\Omega)$, this closure
contains all continuous functions that members of $\mathcal{F}$ can
approximate arbitrarily closely everywhere on $\Omega$, i.e.,
\begin{equation}\label{eq:uniform_closure_definition}
\overline{\mathcal{F}}
=\left\{h\in C(\Omega):
\begin{array}{l}
\text{for every }\varepsilon>0,\ \text{there exists }g\in\mathcal{F}\\
\text{such that }\displaystyle\sup_{\mathbf{i}\in\Omega}
|h(\mathbf{i})-g(\mathbf{i})|<\varepsilon
\end{array}\right\}.
\end{equation}
}
Then
\begin{equation}
\overline{\mathcal{F}_{\mathrm{FTM}}(\mathbf{R})}
\;\subsetneq\;
\overline{\mathcal{F}_{\mathrm{FT\text{-}FiLM}}(\mathbf{R}, V)}
\;=\; C(\Omega).
\end{equation}
\end{theorem}

{   In Theorem~\ref{thm:ftfilm}, we require the readout network class
$\mathcal{N}$ to satisfy the universal approximation property~\cite{UAT1}.
This means that, for every compact rectangle $Q\subset\mathbb{R}^V$,
every continuous target $q\in C(Q)$, and every tolerance $\eta>0$,
there exists a readout $f_{\boldsymbol{\varphi}}\in\mathcal{N}$ such that
\[
\sup_{\mathbf{z}\in Q}|q(\mathbf{z})-f_{\boldsymbol{\varphi}}(\mathbf{z})|<\eta.
\]
Thus the readout can approximate any continuous target uniformly on $Q$;
its parameters and network size may depend on the target and tolerance.}

On the whole, by integrating the proposed observation encoder and FT-FiLM module into the SSM, we obtain our StreamPhy framework, which extends the formulations in Eqs.~(\ref{eq:ssm1})--(\ref{eq:ssm2}), given by:
\begin{mdframed}[backgroundcolor=gray!10,linewidth=0.8pt]
\begin{equation}
	\begin{split}
		\mf{X}_{t} &= \bar{\mathbf{A}} \mf{X}_{t-1} + \bar{\mathbf{b}} \circ \mf{z}_t, \\
		y_t(\mathbf{i}_n) &= f_{\boldsymbol{\varphi}} \!\left(\boldsymbol{\Gamma}_t \mathbf{u}_n + \boldsymbol{\beta}_t \right).
	\end{split}
\end{equation}
\end{mdframed}
 We summarize the overall training procedure in Algorithm~\ref{alg:1} (see Appx.~\ref{algorithm}).

\section{Experiment}
\vspace{-1mm}

\begin{figure}[h]
	\centering
	\includegraphics[width=0.95\linewidth]{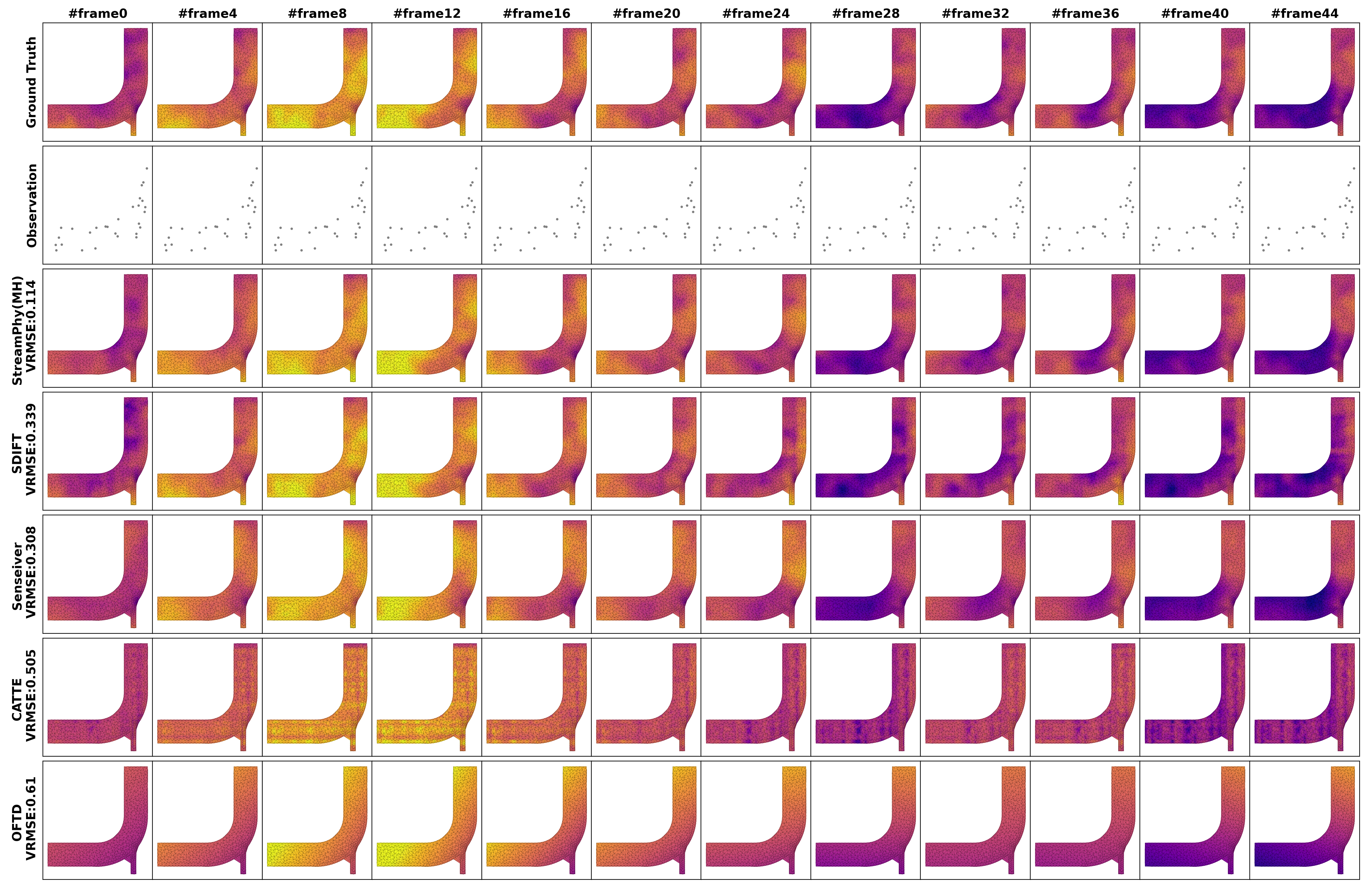}
	\caption{Reconstruction of \textit{Turbulent Flow} dynamics under uniform sampling pattern with \(\rho = 3\%\).}
	\label{fig:tf_plot1}
	\vspace{-0.2in}
\end{figure}

\begin{table*}[t]
	\centering
	\setlength{\tabcolsep}{1.2pt}
	\renewcommand{\arraystretch}{0.95}
	\begin{scriptsize}
		\resizebox{\textwidth}{!}{
			\begin{tabular}{l|cc|cc|cc||cc|cc|cc}
				\toprule
				& \multicolumn{6}{c||}{\textbf{Uniform Sampling}} 
				& \multicolumn{6}{c}{\textbf{Slab Sampling}} \\
				
				\cmidrule(r){2-7} \cmidrule(l){8-13}
				
				& \multicolumn{2}{c}{\textit{Turbulent Flow}} 
				& \multicolumn{2}{c}{\textit{Ocean Sound Speed}} 
				& \multicolumn{2}{c||}{\textit{Active Matter}}
				& \multicolumn{2}{c}{\textit{Turbulent Flow}} 
				& \multicolumn{2}{c}{\textit{Ocean Sound Speed}} 
				& \multicolumn{2}{c}{\textit{Active Matter}} \\
				
				& $\rho=3\%$ & $\rho=5\%$ 
				& $\rho=1\%$ & $\rho=3\%$ 
				& $\rho=1\%$ & $\rho=3\%$
				& $\rho=3\%$ & $\rho=5\%$ 
				& $\rho=1\%$ & $\rho=3\%$ 
				& $\rho=1\%$ & $\rho=3\%$ \\
				\midrule
				
				\multicolumn{13}{l}{\textit{Tensor-based}}\\
				
				LRTFR~\cite{luo2023lowrank}
				& \shortstack{$0.5633$\\$\pm0.072$}
				& \shortstack{$0.3505$\\$\pm0.051$}
				& \shortstack{$0.3453$\\$\pm0.041$}
				& \shortstack{$0.2176$\\$\pm0.029$}
				& \shortstack{$0.3021$\\$\pm0.027$}
				& \shortstack{$0.2582$\\$\pm0.025$}
				& \shortstack{$0.6517$\\$\pm0.125$}
				& \shortstack{$0.4897$\\$\pm0.092$}
				& \shortstack{$0.4075$\\$\pm0.066$}
				& \shortstack{$0.3757$\\$\pm0.040$}
				& \shortstack{$0.9988$\\$\pm0.163$}
				& \shortstack{$0.9537$\\$\pm0.137$} \\
				\addlinespace[4pt]
				
				OFTD~\cite{OFTD2025}
				& \shortstack{$0.5416$\\$\pm0.069$}
				& \shortstack{$0.4810$\\$\pm0.048$}
				& \shortstack{$0.2571$\\$\pm0.031$}
				& \shortstack{$0.1470$\\$\pm0.016$}
				& \shortstack{$0.5542$\\$\pm0.066$}
				& \shortstack{$0.3289$\\$\pm0.027$}
				& \shortstack{$0.5034$\\$\pm0.058$}
				& \shortstack{$0.4901$\\$\pm0.067$}
				& \shortstack{$0.5348$\\$\pm0.049$}
				& \shortstack{$0.4847$\\$\pm0.085$}
				& \shortstack{$0.7387$\\$\pm0.143$}
				& \shortstack{$0.6721$\\$\pm0.061$} \\
				\addlinespace[4pt]
				
				CATTE~\cite{catte2025}
				& \shortstack{$0.6192$\\$\pm0.050$}
				& \shortstack{$0.5916$\\$\pm0.071$}
				& \shortstack{$0.0941$\\$\pm0.007$}
				& \shortstack{$0.0850$\\$\pm0.011$}
				& \shortstack{$0.1581$\\$\pm0.019$}
				& \shortstack{$0.1567$\\$\pm0.021$}
				& \shortstack{$0.6563$\\$\pm0.059$}
				& \shortstack{$0.6257$\\$\pm0.061$}
				& \shortstack{$0.1225$\\$\pm0.024$}
				& \shortstack{$0.1187$\\$\pm0.021$}
				& \shortstack{$0.1960$\\$\pm0.035$}
				& \shortstack{$0.1920$\\$\pm0.044$} \\
				
				\midrule
				
				\multicolumn{13}{l}{\textit{Attention-based}}\\
				
				Senseiver~\cite{santos2023developmentNMI}
				& \shortstack{$0.2779$\\$\pm0.022$}
				& \shortstack{$0.2361$\\$\pm0.014$}
				& \shortstack{$0.1553$\\$\pm0.018$}
				& \shortstack{$0.1248$\\$\pm0.007$}
				& \shortstack{$0.2555$\\$\pm0.018$}
				& \shortstack{$0.2020$\\$\pm0.028$}
				& \shortstack{$0.3052$\\$\pm0.032$}
				& \shortstack{$0.2578$\\$\pm0.029$}
				& \shortstack{$0.2008$\\$\pm0.020$}
				& \shortstack{$0.1768$\\$\pm0.019$}
				& \shortstack{$0.2762$\\$\pm0.042$}
				& \shortstack{$0.2430$\\$\pm0.023$} \\
				
				\midrule
				
				\multicolumn{13}{l}{\textit{Diffusion-based}}\\
				
				SDIFT~\cite{sdift2025}
				& \shortstack{$0.3248$\\$\pm0.052$}
				& \shortstack{$0.2821$\\$\pm0.034$}
				& \shortstack{$0.1466$\\$\pm0.012$}
				& \shortstack{$0.1084$\\$\pm0.008$}
				& \shortstack{$0.2158$\\$\pm0.029$}
				& \shortstack{$0.1564$\\$\pm0.018$}
				& \shortstack{$0.4070$\\$\pm0.053$}
				& \shortstack{$0.4029$\\$\pm0.049$}
				& \shortstack{$0.2713$\\$\pm0.019$}
				& \shortstack{$0.1452$\\$\pm0.023$}
				& \shortstack{$0.4231$\\$\pm0.041$}
				& \shortstack{$0.3985$\\$\pm0.021$} \\
				
				\midrule
				
				\multicolumn{13}{l}{\textit{SSM-based}}\\
				
				StreamPhy (SH)
				& \shortstack{$0.1283$\\$\pm0.013$}
				& \shortstack{$0.0954$\\$\pm0.010$}
				& \shortstack{$0.0862$\\$\pm0.010$}
				& \shortstack{$0.0760$\\$\pm0.009$}
				& \shortstack{$0.0895$\\$\pm0.006$}
				& \shortstack{$0.0833$\\$\pm0.009$}
				& \shortstack{$0.2094$\\$\pm0.037$}
				& \shortstack{$0.1979$\\$\pm0.031$}
				& \shortstack{$0.1324$\\$\pm0.017$}
				& \shortstack{$0.1045$\\$\pm0.015$}
				& \shortstack{$0.1452$\\$\pm0.022$}
				& \shortstack{$0.1278$\\$\pm0.029$} \\
				\addlinespace[4pt]
				
				StreamPhy (MH)
				& \shortstack{$\mathbf{0.0935}$\\$\pm\mathbf{0.011}$}
				& \shortstack{$\mathbf{0.0696}$\\$\pm\mathbf{0.007}$}
				& \shortstack{$\mathbf{0.0628}$\\$\pm\mathbf{0.005}$}
				& \shortstack{$\mathbf{0.0562}$\\$\pm\mathbf{0.008}$}
				& \shortstack{$\mathbf{0.0752}$\\$\pm\mathbf{0.010}$}
				& \shortstack{$\mathbf{0.0732}$\\$\pm\mathbf{0.008}$}
				& \shortstack{$\mathbf{0.1067}$\\$\pm\mathbf{0.011}$}
				& \shortstack{$\mathbf{0.0978}$\\$\pm\mathbf{0.014}$}
				& \shortstack{$\mathbf{0.0692}$\\$\pm\mathbf{0.006}$}
				& \shortstack{$\mathbf{0.0634}$\\$\pm\mathbf{0.006}$}
				& \shortstack{$\mathbf{0.0905}$\\$\pm\mathbf{0.012}$}
				& \shortstack{$\mathbf{0.0866}$\\$\pm\mathbf{0.013}$} \\
				
				\bottomrule
			\end{tabular}
		}
	\end{scriptsize}
	
	\caption{\small Average VRMSEs (mean $\pm$ standard deviation) under two sampling patterns across datasets and observation ratios. The best results are highlighted in bold.}
	\label{table:t1}
	\vspace{-0.1in}
\end{table*}

\textbf{Datasets:}
We evaluate StreamPhy on three  physical datasets. 
(1) \textit{Turbulent Flow}(\url{https://zenodo.org/records/14037782}) , which captures the spatiotemporal evolution of turbulent fluid velocity fields with complex multiscale vortex dynamics and chaotic structures. We extract 600 records of size $48 \times 918 \times 1$, using 500 for training (each consisting of 48 frames with 918 irregular spatial observations) and 100 for testing.
(2) \textit{Ocean Sound Speed}(\url{https://ncss.hycom.org/thredds/ncss/grid/GLBy0.08/expt_93.0/ts3z/dataset.html}) , which contains sound speed field measurements in the Pacific Ocean. We extract 1000 records of size $24 \times 5 \times 38 \times 76$ (24 frames), using 950 for training and 50 for testing; during training, only 10\% of spatial points are used to simulate irregular observations.
(3) \textit{Active Matter}(\url{https://polymathic-ai.org/the_well/datasets/active_matter/})
, which models the dynamics of rod-like active particles in a Stokes fluid via continuum simulations. We extract 928 records of size $24 \times 256 \times 256$ (24 frames), using 900 for training and 28 for testing; similarly, 10\% of spatial points are used during training.

\textbf{Baselines and Settings:}
We conduct  comprehensive evaluations of StreamPhy against state-of-the-art approaches for physical field reconstruction and tensor-based methods.
(1) \textbf{SDIFT}~\citep{sdift2025}, a diffusion framework in functional Tucker space that reconstructs full-field multidimensional physical dynamics from irregular sparse observations by modeling latent functional representations and their temporal evolution; 
(2) \textbf{Senseiver}~\citep{santos2023developmentNMI}, an attention-based method that embeds sparse sensor measurements into a unified latent space for efficient multidimensional field reconstruction; 
(3) \textbf{LRTFR}~\citep{luo2023lowrank}, a low-rank functional Tucker model leveraging factorized neural representations for tensor decomposition; 
(4) \textbf{OFTD}~\citep{OFTD2025}, an online tensor decomposition approach incorporating implicit neural representations into CP model to continuously model streaming data while updating without forgetting past information; 
(5) \textbf{CATTE}~\citep{catte2025}, a functional temporal tensor decomposition framework that encodes continuous indices with Fourier features and neural ODEs, while automatically adapting model complexity through sparsity-inducing priors.  We evaluate all methods on the task of temporal physical field reconstruction under two sampling patterns, with observation ratios $\rho \in \{1\%, 3\%, 5\%\}$. For the uniform sampling pattern, observations are evenly distributed across the entire spatiotemporal domain. At each time step $t_i$, a constant sampling ratio $\rho$ is maintained, and observations are acquired sequentially in temporal order.  For the slab sampling pattern, we consider a non-uniform observation scheme~\cite{bttd}\footnote{Slab sampling refers to observing entire tensor slices at once, as in Unmanned Aerial Vehicle (UAV) sensing, where a drone captures full field slices over a trajectory or time window rather than pointwise measurements.}, while preserving the same sampling ratio $\rho$ at each time step $t_i$.
Following~\cite{sdift2025}, we evaluate performance using the variance-scaled root mean squared error (VRMSE; definition refers to Appx.~\ref{app:vrmse}), a scale-invariant measure of reconstruction accuracy. Each experiment is repeated 10 times, and we report the mean test error. All the implementation details of StreamPhy and other baselines are provided in Appx.~\ref{app:implementation_detail}.
\begin{figure}[t]
	\centering
	\includegraphics[width=0.95\linewidth]{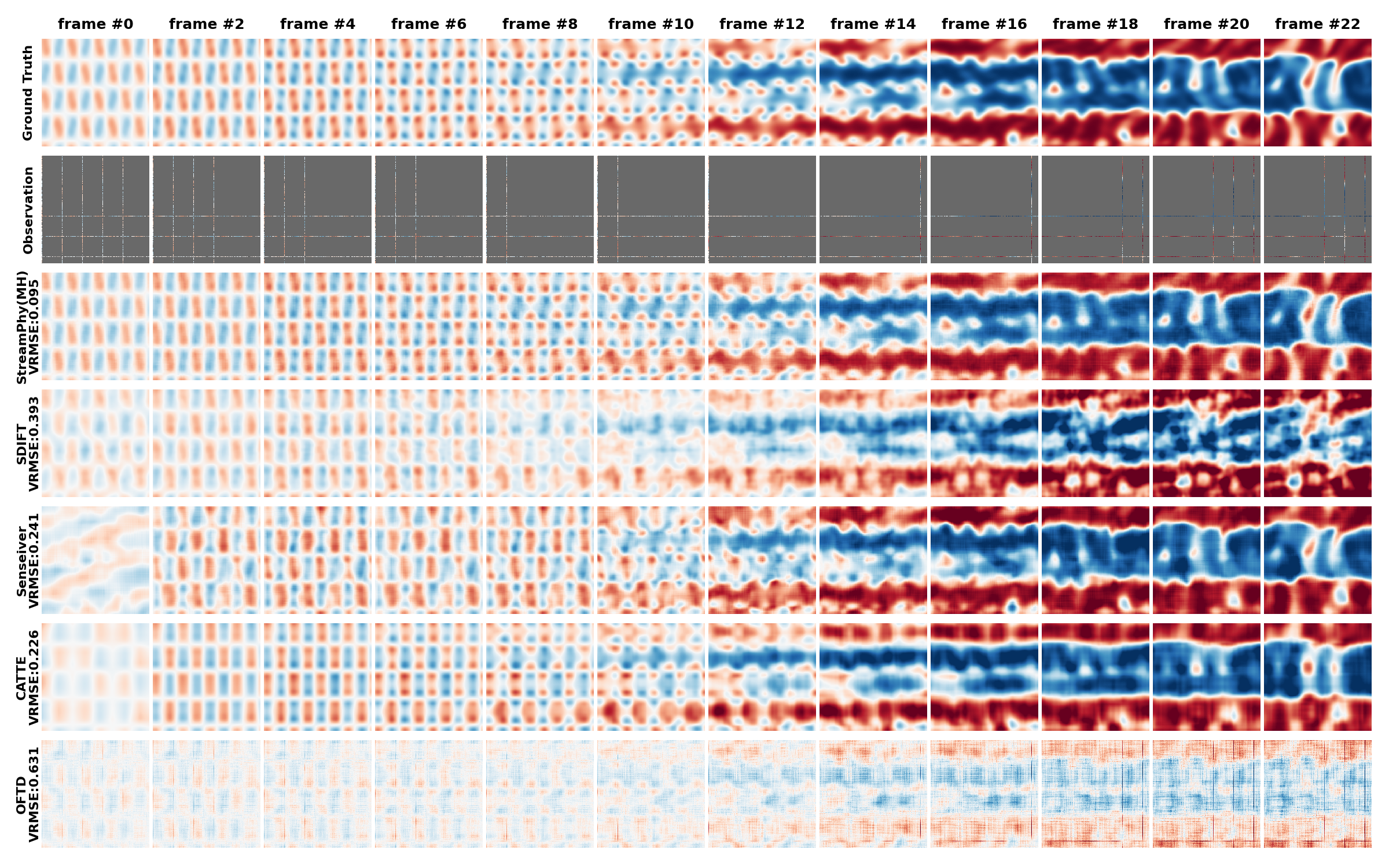}
	\caption{Reconstruction of \textit{Active Matter} dynamics under slab sampling with \(\rho = 3\%\).}
	\label{fig:am_plot2}
\end{figure}

\textbf{Main Evaluation Results:} Note that only OFTD and our StreamPhy support online inference, whereas all other methods require the full temporal sequence before inference. We evaluate two StreamPhy variants: (i) StreamPhy with a single-head observation encoder and the masking strategy in Eq.~\ref{eq:cross_att} (StreamPhy(SH)), (ii) StreamPhy with a multi-head observation encoder and the masking strategy (StreamPhy(MH)). As reported in Tab.~\ref{table:t1}, two variants of StreamPhy  outperform all baselines under both sampling patterns. StreamPhy with a multi-head observation encoder outperforms its single-head counterpart by  effectively captures diverse dependencies across various subspaces.  Qualitative results in Fig.~\ref{fig:tf_plot1} further show that StreamPhy produces more accurate and visually  reconstructions of turbulent flow in an online manner. At the same observation rate, slab sampling is more challenging, and all baselines exhibit a noticeable performance drop (especially SDIFT). This is because it yields spatially concentrated and less informative measurements, which leads to poor coverage, ill-conditioned reconstruction, and stronger dependence on long-range dependencies. As shown in Fig.~\ref{fig:am_plot2}, StreamPhy  effectively alleviates these difficulties and maintains high reconstruction accuracy even though such pattern is absent during training, demonstrating strong robustness. More practical sensing conditions like real-world sensor noise, sensor drift and etc. are considered in Appx.~\ref{app:add_conditions}.  Additional reconstruction visualizations are provided in Appx.~\ref{app:add_results}. Overall, these results validate the effectiveness of StreamPhy.

\textbf{Ablation study:}
We conduct an ablation study to evaluate the contributions of the proposed SSM framework, the masking strategy under varying sampling patterns and sparsity levels, and the expressivity gain of FT-FiLM over FTM. We compare the full StreamPhy with three variants: (i) \textit{w.o. SSM}, which removes the state space model (implemented by setting all $\mathbf{X}_t$ in Eq.~\ref{eq:gamma_beta} to zero), (ii) \textit{w.o. mask}, which disables the masking mechanism in the observation encoder, and (iii) \textit{with FTM}, which replaces FT-FiLM with FTM. All variants share the same multi-head encoder. As shown in Table~\ref{table:t2_ab}, removing the SSM leads to a drastic performance degradation across all datasets and settings, underscoring its critical role in capturing long-range temporal dependencies from streaming observations. Disabling the masking strategy also results in consistent performance drops, particularly under slab sampling, highlighting its importance in handling irregular and partially observed inputs and in improving both accuracy and robustness. Replacing FT-FiLM with FTM further degrades performance, demonstrating the superior expressivity of FT-FiLM for high-fidelity reconstruction. Overall, StreamPhy achieves the best performance across all settings, indicating that all components are essential.

\begin{table*}
\centering
\setlength{\tabcolsep}{0.8 pt}
\renewcommand{\arraystretch}{0.9}
\begin{scriptsize}
\begin{tabular}{l|cc|cc|cc||cc|cc|cc}
\toprule
& \multicolumn{6}{c||}{\textbf{Uniform Sampling}} 
& \multicolumn{6}{c}{\textbf{Slab Sampling}} \\

\cmidrule(r){2-7} \cmidrule(l){8-13}

& \multicolumn{2}{c}{\textit{Turbulent Flow}} 
& \multicolumn{2}{c}{\textit{Ocean Sound Speed}} 
& \multicolumn{2}{c||}{\textit{Active Matter}}

& \multicolumn{2}{c}{\textit{Turbulent Flow}} 
& \multicolumn{2}{c}{\textit{Ocean Sound Speed}} 
& \multicolumn{2}{c}{\textit{Active Matter}} \\

& $\rho=3\%$ & $\rho=5\%$ 
& $\rho=1\%$ & $\rho=3\%$ 
& $\rho=1\%$ & $\rho=3\%$

& $\rho=3\%$ & $\rho=5\%$ 
& $\rho=1\%$ & $\rho=3\%$ 
& $\rho=1\%$ & $\rho=3\%$ \\
\midrule

StreamPhy 
&  \textbf{0.0935} & \textbf{0.0696} &\textbf{0.0628}  & \textbf{0.0562} & \textbf{0.0752} & \textbf{0.0732} 
& \textbf{0.1067} & \textbf{0.0978} & \textbf{0.0692} & \textbf{0.0634} & \textbf{0.0905} & \textbf{0.0866} \\

StreamPhy w.o. SSM
& 0.8537 & 0.8501& 0.3935 & 0.3923  & 0.9023 & 0.8687
& 0.8598 & 0.8556 & 0.3966 & 0.3954 & 0.9115 & 0.9107 \\

StreamPhy w.o. mask 
& 0.2420 & 0.1789 & 0.0899 & 0.0786  & 0.0898 & 0.0786 
& 0.2841 & 0.2643 & 0.2030 & 0.1768 & 0.2401 & 0.1511 \\

StreamPhy with FTM 
& 0.1435 & 0.1134 & 0.1054 & 0.1034 & 0.1279 & 0.1229 
& 0.1731 & 0.1574 & 0.1202 & 0.1154 & 0.1522 & 0.1454 \\

\bottomrule
\end{tabular}
\end{scriptsize}
\caption{\small Ablation study on the effectiveness of the introduced SSM framework, the proposed masking strategy and FT-FiLM. All models use the multi-head observation encoder.}
\label{table:t2_ab}
\vspace{-0.2in}
\end{table*}

    \begin{table}
        \centering
        \begin{scriptsize}
        \setlength{\tabcolsep}{3 pt}
            \begin{tabular}{lccc|ccc|ccc}
                \toprule
                  & \multicolumn{3}{c}{\textit{Turbulent Flow}}  & \multicolumn{3}{c}{\textit{Ocean Sound Speed}} &  \multicolumn{3}{c}{\textit{Active Matter}}  \\ 
                     {Methods} & $\rho=3\%$ & $\rho=5\%$ &{$\#$Para.} & $\rho=1\%$ & $\rho=3\%$ & {$\#$Para.}& $\rho=1\%$ & $\rho=3\%$ &{$\#$Para.} \\
                    \midrule
                    %{$\#$Parameter (M)} &  \multicolumn{2}{c}{12}  &{0}  &{0} &{0} &{0} \\
                    {StreamPhy (MH)} & {\textbf{0.0419s}} &{\textbf{0.0448s}} & 10.3M & {\textbf{0.0325s}} &{\textbf{0.0434s}} &9.94M & {\textbf{0.0479s}}  &{\textbf{0.0498s}} & 11M\\
                    {SDIFT~\cite{sdift2025}} &{{5.14s}} &{{5.21s}}& 15M  &{{0.84s}}  &{{0.89s}}& 15M &{{1.31s}} &{{1.42s}} & 12M \\
                \bottomrule
            \end{tabular}
        \end{scriptsize}
        \caption{\small Average inference speed (in seconds) for reconstruction with different observation ratios on uniform sampling patterns on a single record.}
        \label{table:samplingspeed}
        \vspace{-1mm}
    \end{table}
\textbf{Inference Speed:} We compare the inference speed of our method with SDIFT~\cite{sdift2025} on an NVIDIA RTX 4090 GPU (24\,GB memory) on a single record. The results are reported in Tab.~\ref{table:samplingspeed}. It is evident that our method  outperforms SDIFT across all settings, achieving a speedup of approximately $20\times$--$100\times$.
This substantial acceleration arises from two main aspects. 
1) StreamPhy adopts an end-to-end architecture that directly maps observations to sequential reconstruction of the full field, whereas SDIFT~\cite{sdift2025} relies on fully observed temporal data and performs multi-step posterior sampling. In particular, its MPDPS mechanism introduces a Gaussian process with per-frame complexity $O(T^3)$, where $T$ denotes the number of frames, resulting in an overall complexity of $O(T^4)$. In contrast, our method scales linearly with $T$, i.e., $O(T)$. 
2) Benefiting from the strong expressive power of the proposed FT-FiLM module, StreamPhy employs significantly more compact latent representations than SDIFT (e.g., $64$ vs.\ $2304$ ($48\times48$) for the active matter dataset), which further contributes to its superior inference efficiency.

%\textbf{Robustness against Noise:} We evaluated SDIFT with DPS and MPDPS against three noise types: Gaussian, Poisson and Laplacian, with different variance levels, and results are at Tab.~\ref{table:noise_Ocean Sound Speed} in Appx.~\ref{app:noise_rob}, which demonstrates the robustness introduced by proposed MSDPS module.

\section{Conclusion and Limitations}
\vspace{-3mm}
In this work, we presented \textbf{StreamPhy}, an end-to-end framework for streaming inference of continuous spatiotemporal dynamics from sparse observations. By proposing a data-adaptive observation encoder that is robust to missing patterns, together with the expressive FT-FiLM module, and integrating both into a HiPPO-based  state-space model, StreamPhy enables efficient and accurate streaming inference, offering a scalable solution for engineering applications. Experiments across multiple physical systems demonstrate consistent gains in both accuracy and efficiency under diverse challenging sampling patterns. Limitations of this work are provided in Appx.~\ref{app:limitations}.
%Future work includes uncertainty quantification and cross-domain generalization.

\section{Acknowledgement}
This work was supported in part by the National Natural Science Foundation of China under Grant 62622127 and 62371418, and supported by the Zhejiang Provincial Natural Science Foundation of China under Grant No. LR26F010004.

\bibliographystyle{unsrt}
\bibliography{references}

%%%%%%%%%%%%%%%%%%%%%%%%%%%%%%%%%%%%%%%%%%%%%%%%%%%%%%%%%%%%
\newpage
\appendix

\section{Proof of Theorem 1}
\label{app:sec:proof}

\begin{proof}
	Both model classes consist of continuous functions on $\Omega$:
	functional Tucker models are finite sums of products of continuous
	functions, whereas FT-FiLM models are compositions of continuous
	factor functions, an affine map, and a continuous network.
	Since continuity is preserved under uniform convergence, the uniform
	closures of both classes are contained in $C(\Omega)$.
	  {For background on basic topological terminology, including
	closed sets and closure, see~\cite[Appendix~A]{boyd2004convex}.
	Here, closure is taken with respect to the uniform norm on $C(\Omega)$.}
	
	We first show that
	$\mathcal{F}_{\mathrm{FT\text{-}FiLM}}(\mathbf{R},V)$
	is dense in $C(\Omega)$  {, i.e., for every continuous target
	$h\in C(\Omega)$ and every tolerance $\varepsilon>0$, there exists
	$g\in\mathcal{F}_{\mathrm{FT\text{-}FiLM}}(\mathbf{R},V)$ such that
	$\sup_{\mathbf{i}\in\Omega}|h(\mathbf{i})-g(\mathbf{i})|<\varepsilon$.
	Thus every continuous function can be approximated arbitrarily accurately
	everywhere on $\Omega$, without requiring exact representation}.
	Fix $h\in C(\Omega)$ and $\varepsilon>0$.
	Since $R_k\geq 1$ for every $k$, choose
	\[
	\mathbf{u}^{k}(i_k)
	=
	(i_k,0,\ldots,0)^{\top}\in\mathbb{R}^{R_k},
	\qquad k=1,\ldots,K.
	\]
	The condition $V\geq K$ allows us to choose
	$\mathbf{W}_g\in\mathbb{R}^{V\times S}$ and
	$\mathbf{w}_b=\mathbf{0}$ such that
	\[
	\mathbf{W}_g\mathbf{u}_{\mathbf{i}}+\mathbf{w}_b
	=
	(i_1,\ldots,i_K,0,\ldots,0)^{\top}
	=: \iota(\mathbf{i}).
	\]
	Let
	\[
	D=\Omega\times[-1,1]^{V-K},
	\qquad
	\widetilde{h}(\mathbf{z})=h(z_1,\ldots,z_K),
	\quad \mathbf{z}\in D,
	\]
	with $D=\Omega$ when $V=K$.
	Then $D$ is compact, $\widetilde{h}\in C(D)$, and
	$\widetilde{h}(\iota(\mathbf{i}))=h(\mathbf{i})$.
	{ Applying the universal approximation property explained immediately after
	Theorem~\ref{thm:ftfilm} (see~\cite{UAT1})
	with $Q=D$, $q=\widetilde{h}$, and
	$\eta=\varepsilon$, we obtain}
	$f_{\boldsymbol{\varphi}}\in\mathcal{N}$ such that
	\[
	\sup_{\mathbf{z}\in D}
	\bigl|\widetilde{h}(\mathbf{z})
	-f_{\boldsymbol{\varphi}}(\mathbf{z})\bigr|
	<\varepsilon.
	\]
	Since $\iota(\Omega)\subseteq D$, it follows that
	\[
	\sup_{\mathbf{i}\in\Omega}
	\left|
	h(\mathbf{i})
	-
	f_{\boldsymbol{\varphi}}
	\bigl(\mathbf{W}_g\mathbf{u}_{\mathbf{i}}+\mathbf{w}_b\bigr)
	\right|
	<\varepsilon.
	\]
	As $h$ and $\varepsilon$ were arbitrary, we conclude that
	\[
	\overline{\mathcal{F}_{\mathrm{FT\text{-}FiLM}}(\mathbf{R},V)}
	=
	C(\Omega).
	\]
	Consequently,
	\[
	\overline{\mathcal{F}_{\mathrm{FTM}}(\mathbf{R})}
	\subseteq
	C(\Omega)
	=
	\overline{\mathcal{F}_{\mathrm{FT\text{-}FiLM}}(\mathbf{R},V)}.
	\]
	
	To establish strictness,   {we need to find at least one
	continuous function that belongs to the FT-FiLM closure
	$\overline{\mathcal{F}_{\mathrm{FT\text{-}FiLM}}(\mathbf{R},V)}$ but not to the
	fixed-rank FTM closure $\overline{\mathcal{F}_{\mathrm{FTM}}(\mathbf{R})}$,
	i.e., one that fixed-rank FTM cannot approximate
	arbitrarily closely under the uniform norm.} Consider the continuous function
	\[
	h_{\star}(\mathbf{i})=\exp(i_1i_2).
	\]
	When $K>2$, fix the remaining coordinates and suppress these fixed
	arguments in the notation below.
	Set
	\[
	r=\min\{R_1,R_2\},
	\qquad M=r+1.
	\]
	Because $\Omega_1$ and $\Omega_2$ have nonempty interior, we may choose
	equally spaced points
	\[
	\alpha_j=a+(j-1)\delta\in\Omega_1,
	\qquad j=1,\ldots,M,
	\qquad \delta>0,
	\]
	and distinct points
	$\beta_1<\cdots<\beta_M$ in $\Omega_2$.
	
	For any $h\in\mathcal{F}_{\mathrm{FTM}}(\mathbf{R})$, its restriction
	to this slice can be written as
	\[
	h(i_1,i_2)
	=
	\mathbf{u}^{1}(i_1)^{\top}
	\mathbf{B}
	\mathbf{u}^{2}(i_2),
	\qquad
	\mathbf{B}\in\mathbb{R}^{R_1\times R_2},
	\]
	where $\mathbf{B}$ is the Tucker core when $K=2$, or its contraction
	with the factor vectors at the fixed coordinates when $K>2$.
	The sampled matrix therefore admits the factorization
	\[
	\mathbf{H}
	:=
	\bigl[h(\alpha_j,\beta_\ell)\bigr]_{j,\ell=1}^{M}
	=
	\mathbf{A}\mathbf{B}\mathbf{C}^{\top},
	\]
	where
	\[
	\mathbf{A}_{jp}=u^{1}_{p}(\alpha_j),
	\qquad
	\mathbf{C}_{\ell q}=u^{2}_{q}(\beta_\ell).
	\]
	Hence
	\[
	\operatorname{rank}(\mathbf{H})
	\leq \min\{R_1,R_2\}
	=r<M,
	\]
	so $\det(\mathbf{H})=0$.
	
	On the other hand,   {recall the target function
	$h_{\star}(\mathbf{i})=\exp(i_1i_2)$ introduced above to establish strictness.
	Its sampled matrix $\mathbf{M}$} satisfies
	\[
	\mathbf{M}_{j\ell}
	=
	e^{\alpha_j\beta_\ell}
	=
	e^{a\beta_\ell}
	\bigl(e^{\delta\beta_\ell}\bigr)^{j-1}.
	\]
	Thus $\mathbf{M}$ is a Vandermonde matrix with nonzero column
	scaling factors, and
	\[
	\det(\mathbf{M})
	=
	\left(\prod_{\ell=1}^{M}e^{a\beta_\ell}\right)
	\prod_{1\leq p<q\leq M}
	\left(e^{\delta\beta_q}-e^{\delta\beta_p}\right)
	\neq 0.
	\]
	
{ 
Since the determinant is continuous and $\det(\mathbf{M})\neq0$,
there exists $\rho>0$ such that every matrix
$\mathbf{A}\in\mathbb{R}^{M\times M}$ satisfying
\[
\max_{1\leq j,\ell\leq M}|A_{j\ell}-M_{j\ell}|<\rho
\]
has nonzero determinant. Thus every singular matrix differs from
$\mathbf{M}$ by at least $\rho$ in some entry.
For any $h\in\mathcal{F}_{\mathrm{FTM}}(\mathbf{R})$, its sampled matrix
$\mathbf{H}=[h(\alpha_j,\beta_\ell)]_{j,\ell=1}^{M}$ is singular, as shown above.
Consequently,
\[
\sup_{\mathbf{i}\in\Omega}|h(\mathbf{i})-h_{\star}(\mathbf{i})|
\geq
\max_{1\leq j,\ell\leq M}
|h(\alpha_j,\beta_\ell)-h_{\star}(\alpha_j,\beta_\ell)|
\geq\rho.
\]
Here the sampled values are evaluated on the fixed slice when $K>2$.
This positive lower bound holds for every fixed-rank FTM function.
By the definition of closure in Eq.~\eqref{eq:uniform_closure_definition},
it follows that
\[
h_{\star}\notin\overline{\mathcal{F}_{\mathrm{FTM}}(\mathbf{R})}.
\]
}
	Together with $h_{\star}\in C(\Omega)$ and the density of FT-FiLM,
	this proves
	\[
	\overline{\mathcal{F}_{\mathrm{FTM}}(\mathbf{R})}
	\subsetneq
	\overline{\mathcal{F}_{\mathrm{FT\text{-}FiLM}}(\mathbf{R},V)}
	=
	C(\Omega).
	\]
\end{proof}

\begin{remark}[Scope of Theorem~\ref{thm:ftfilm}]
Theorem~\ref{thm:ftfilm} establishes a \emph{qualitative} expressivity gap:
fixed-rank functional Tucker models cannot uniformly approximate every
continuous function on $\Omega$, whereas FT-FiLM can. The density of FT-FiLM
is driven by the readout network $f_{\boldsymbol{\varphi}}$; thus the theorem
shows that decoupling the rank of the latent factors from the capacity of
the readout, which is a key feature of FT-FiLM, removes the rank-induced
approximation barrier of FTM. 
\end{remark}

\section{Details of Modulation Computation}
\label{app:detail}
Note that, 
\begin{equation}
	\boldsymbol{\Gamma}_t = f_{\boldsymbol{\omega}_1}(\mathbf{X}_t, \mathbf{z}_t), 
	\quad 
	\boldsymbol{\beta}_t = f_{\boldsymbol{\omega}_2}(\mathbf{X}_t, \mathbf{z}_t),
\end{equation}

Here, we provide detailed computations of modulation parameters $\boldsymbol{\Gamma}_t$ and $\boldsymbol{\beta}_t$.
\begin{equation}
	f_{\boldsymbol{\omega}_1}(\mathbf{X}_t, \mathbf{z}_t) = \text{Reshape}(\text{MLP}(\text{Concat}([\text{vec}(\mf{X}_t),\mf{z_t}]))).
	\label{eq:detail1}
\end{equation}
As shown in Eq.~(\ref{eq:detail1}), 
$f_{\boldsymbol{\omega}_1}(\cdot, \cdot):\mathbb{R}^{L \times P} \times \mathbb{R}^{P} \to \mathbb{R}^{ V \times \sum_{k=1}^{K}R_k }$ 
maps a matrix together with a vector to a structured modulation matrix. Specifically, the input observation $\mathbf{X}_t \in \mathbb{R}^{L \times P}$ is first vectorized as $\mathrm{vec}(\mathbf{X}_t) \in \mathbb{R}^{LP}$ and concatenated with the conditioning variable $\mathbf{z}_t \in \mathbb{R}^{P}$ to form a unified representation. This concatenated feature is then processed by a multilayer perceptron to capture nonlinear interactions between spatial observations and conditioning signals. The resulting output is subsequently reshaped via $\mathrm{Reshape}(\cdot)$ into a tensor of dimension $ V \times  \sum_{k=1}^{K} R_k $, where each row corresponds to a feature-wise modulation coefficient over the $V$ latent channels. 

In this way, $f_{\boldsymbol{\omega}_1}$ generates a global modulation dictionary that parameterizes feature-wise scaling factors conditioned jointly on $\mathbf{X}_t$ and $\mathbf{z}_t$, enabling expressive cross-feature and cross-condition interactions within the FT-FiLM framework.

Also, we have
\begin{equation}
	f_{\boldsymbol{\omega}_2}(\mathbf{X}_t, \mathbf{z}_t)=\text{MLP}(\text{Concat}([\text{vec}(\mf{X}_t),\mf{z_t}]))
	\label{eq:detail2}
\end{equation}
As shown in Eq.~(\ref{eq:detail2}), 
$f_{\boldsymbol{\omega}_2}(\cdot, \cdot):\mathbb{R}^{L \times P} \times \mathbb{R}^{P} \to \mathbb{R}^{V}$ 
maps the joint representation of the observed field $\mathbf{X}_t$ and the conditioning variable $\mathbf{z}_t$ to a feature-wise bias vector. In particular, $\mathbf{X}_t$ is first vectorized as $\mathrm{vec}(\mathbf{X}_t) \in \mathbb{R}^{LP}$ and concatenated with $\mathbf{z}_t \in \mathbb{R}^{P}$, forming a unified input embedding. This embedding is then passed through a multilayer perceptron to capture nonlinear dependencies between spatial observations and the conditioning signal. The output is a $V$-dimensional vector, where each entry corresponds to the additive shift applied to the corresponding latent feature channel.

Together with the scaling parameters produced by $f_{\boldsymbol{\omega}_1}$, this bias term enables a complete feature-wise affine modulation in the FT-FiLM layer, thereby enhancing the expressiveness of conditional field representations.
\section{Algorithm}  \label{algorithm}
In our setting, the training data is irregular and sparse observations sampled from $B$ batches of homogeneous physical dynamics at arbitrary timesteps.

We summarize the training process of the StreamPhy in Algorithm\ref{alg:1}.
\begin{algorithm}[h!]
  \caption{Training process}
  \label{alg:1}
  \begin{algorithmic}[1]
    \REQUIRE Batches of training data $\{\{\mathcal{O}_{t_m}^b\}_{m=1}^{M}\}_{b=1}^{B}$, time steps $\mathcal{T}=\{t_1,\dots,t_M\}$. Initialize the latent state $\mf{X}_0 = \boldsymbol{0}$.

    \WHILE{not convergence}
    \FOR{$t=t_1,\dots,t_M$}
      \STATE  Sample a rate from $U[0.1,1]$ and generate mask to randomly mask-out observation points according to the sampled rate.
  \STATE Obtain observation latent $\mf{z}_t$ from observation encoder with the mask through Eq.~(\ref{eq:cross_att}).
  \STATE Compute discretized matrices $\bar{\mf{A}}, \bar{\mf{b}}$ using Eq.~(\ref{eq:hippo}).
      \STATE Update the current latent state using Eq.~(\ref{eq:latent_state}).
      \STATE Reconstruct full field using Eq.~\ref{eq:read_out}.
      \STATE Take gradients on the RMSE of reconstruction errors using Adam~\cite{adam} optimizer.
    \ENDFOR
   \ENDWHILE
    \STATE \textbf{return} Well-trained model parameters.
  \end{algorithmic}
\end{algorithm}

\section{Related Work}
\label{app:related_work}
Tensor-based methods leverage low-rank structures to recover high-dimensional data from sparse observations. 
Recent progress in temporal tensor modeling~\cite{NONFAT, wang2024dynamictensor, bctt, fang2023streamingSFTL} extends classical discrete formulations to continuous-time regimes, while functional tensor approaches~\cite{fang2023functional, luo2023lowrank, chertkov2022optimization_tt, chen2025generalized} represent tensor modes as continuous functions with structured latent parameterizations. 
More recent studies~\cite{catte2025, OFTD2025} further incorporate Neural ODEs or streaming mechanisms to model temporal dynamics.  Despite these advances, such methods typically operate on individual instances, limiting their ability to exploit shared structure and generalize across heterogeneous data.

To address this limitation, generative modeling has recently emerged as an alternative paradigm. 
A growing body of work~\cite{chung2023diffusion, shu2023physics, shysheya2024conditional, huang2024diffusionpde, li2024learning_nmi2} leverages pre-trained diffusion models and demonstrates strong performance across diverse physical modeling tasks. 
However, these approaches generally rely on regularly sampled, well-structured data, restricting their applicability in realistic settings with sparse and irregular observations. 
To mitigate this issue,~\cite{du2024conditional, sdift2025} project irregular observations into a regularized latent space and model the temporal evolution of latent distributions. 
At inference time, full-field reconstruction is achieved by incorporating likelihood-based gradients into the score function as guidance. 
Nevertheless, this two-stage pipeline is prone to accumulated approximation errors and reconstruction artifacts, while incurring significant computational overhead due to repeated forward passes and gradient evaluations.

Structured state space models (SSMs) have recently emerged as an efficient and principled framework for long-range sequence modeling. 
Early work such as HiPPO~\cite{hippo} introduces a continuous-time memory mechanism based on optimal polynomial projections, providing a theoretically grounded way to compress streaming signals into finite-dimensional states. 
Building on this foundation, S4~\cite{lssl} develops a scalable parameterization of linear state space systems, enabling efficient modeling of long sequences through structured transition matrices and fast convolutional implementations. 
Subsequent work further unifies recurrent, convolutional, and continuous-time formulations via linear state space layers~\cite{s4}, offering a flexible framework that bridges different sequence modeling paradigms. 
More recently, selective state space models such as Mamba~\cite{mamba} introduce input-dependent state transitions, significantly enhancing modeling capacity while retaining linear-time complexity. 
Despite their strong performance, these approaches typically assume regularly sampled inputs and are not explicitly designed for sparse and irregular observations in spatiotemporal reconstruction tasks.

\section{Implementation details}
\label{app:implementation_detail}
All the methods are implemented with PyTorch \cite{paszke2019pytorch} and trained  using Adam \cite{adam} optimizer with the  learning rate tuned from  $\{5e^{-4}, 1e^{-3}, 5e^{-3}, 1e^{-2}\}$.

For StreamPhy, we set $L=32, h=4, D=512, P=64$ for all datasets. While $\mathbf{R}=(64,64)$ for Turbulent Flow, $\mathbf{R}=(3,15,30)$ for Ocean Sound Speed and $\mathbf{R}=(128,128)$ for Active Matter.

For OFTD, we adopt a two-stage online CP decomposition scheme with an initial window of $t_{\mathrm{initial}}=5$ frames and an online update step of $\Delta t=1$. We parameterize the temporal factor using a two-layer SIREN followed by a linear output layer, with $\omega_0=0.3$. We train the initial fitting stage for 4000 iterations and each subsequent online update for 100 iterations, while holding out 10\% of the observed entries for validation and using early stopping with patience 10. Across datasets, we set the CP rank and hidden width to $(256,256)$ for \textit{active matter} and $(64,192)$ for both \textit{SSF} and \textit{turbulent flow}.

For {LRTFR, we parameterize each mode-wise latent factor using a two-layer SIREN with a linear output layer, and reconstruct the full field through a low-rank core tensor. We train the model using Adam with learning rate $10^{-4}$ and weight decay $0.1$ for 3001 iterations, again reserving 10\% of the observed entries for validation. The hidden width is set to 192 for \textit{active matter} and 128 for both \textit{SSF} and \textit{turbulent flow}. The mode-wise downsampling factors are set to $\{8,8,4\}$, $\{2,8,8,8\}$, and $\{8,4\}$ for the three datasets, respectively, and the rank of each mode is determined adaptively as the corresponding mode size divided by its downsampling factor.

For Senseiver, we use 32 spatial Fourier bands for positional encoding in all cases. For \textit{active matter}, we additionally enable time as an extra input dimension, and set the encoder preprocessing width to 96, the number of latent tokens to 12, the encoder and decoder latent widths to 48, the number of encoder layers to 4, the numbers of cross-attention and self-attention heads to 4, and the number of self-attention layers per block to 3. For \textit{SSF} and \textit{turbulent flow}, we use the default configuration with encoder preprocessing width 64, 4 latent tokens, encoder and decoder latent widths of 16, 3 encoder layers, 2 encoder cross-attention heads, 2 encoder self-attention heads, and 1 decoder cross-attention head.

 For CATTE,  we set  the ODE state dimension $J=15$ and  the initial number of components of the factor trajectories $R = 15$.
 We used two hidden layers for velocity network, with the layer width chosen from $\{128, 256, 512\}$. For LRTFR, we  used two hidden layers with  layer width chosen from $\{128, 256, 512\}$ to  parameterize the latent function of each mode. We varied $R$  from $\{16,32\}$ for all baselines.
For Senseiver, we used 128 channels in both the encoder and decoder, a sequence size of 256 for the $Q_{in}$ array, and set the size of the linear layers in the encoder and decoder to 128.

For SDIFT, we first apply a functional Tucker model to decompose the tensor into factor functions and a core sequence. Each factor function is parameterized by a three-layer MLP, where each layer contains 1024 neurons and uses the sine activation function. The core sizes are set to $48 \times 48$, $3 \times 12 \times 12$, and $48 \times 48$ for the \textit{Turbulent Flow}, \textit{Ocean Sound Speed}, and \textit{Active Matter} datasets, respectively. These hyperparameters are carefully selected to achieve optimal performance.

\subsection{Definition of Variance-scaled Root Mean Squared Error}
\label{app:vrmse}
Let $\{\hat{y}_i\}_{i=1}^N$ and $\{y_i\}_{i=1}^N$ denote the predicted and ground-truth entry, respectively. Assume that there are $N$ points in total. The Variance-scaled Root Mean Squared Error (VRMSE) is defined as
\begin{equation}
    \mathrm{VRMSE} = \frac{\displaystyle\sqrt{\frac{1}{N}\sum_{i=1}^{L} (\hat y_i - y_i)^2}}
     {\displaystyle\sqrt{\frac{1}{N}\sum_{i=1}^{L} (y_i - \bar y)^2}},
\end{equation}
where $\bar y$ is the mean of all points.

\section{Additional Experiment Results}
\label{app:add_conditions}

Here, we report additional experimental results to further evaluate the robustness and computational efficiency of StreamPhy under practical sensing conditions. 
Specifically, we consider real-world sensor noise, sensor drift and distribution shifts, missing frames, irregular temporal arrivals, and delayed sensor measurements. 
We also compare the model size, training cost, and inference efficiency of StreamPhy with all baseline methods across different datasets and observation densities. 
Unless otherwise specified, the robustness experiments are conducted on the Active Matter Dataset under uniform spatial sampling.

\paragraph{Robustness to real-world sensor noise.}
We first evaluate the robustness of StreamPhy under different types and levels of measurement noise.
Gaussian noise is used to model additive random measurement errors commonly introduced by sensors and environmental disturbances.
Laplacian noise is used to simulate impulsive noise, such as occasional sensor spikes, while Poisson noise is used to simulate signal-dependent noise commonly observed in photon-counting and event-based sensing systems.

Table~\ref{tab:noise_robustness} summarizes the reconstruction performance under different noise configurations.
StreamPhy consistently maintains low reconstruction errors across different noise types and noise levels, demonstrating its robustness to corrupted sensor observations.

\begin{table}[t]
	\centering
	\caption{
		Reconstruction performance under different types and levels of sensor noise on the Active Matter Dataset with uniform sampling ratio $\rho=1\%$.
		Results are reported as mean $\pm$ standard deviation.
	}
	\label{tab:noise_robustness}
	\small
	\resizebox{\textwidth}{!}{
		\begin{tabular}{lcccccc}
			\toprule
			\textbf{Noise Configuration}
			& \textbf{StreamPhy}
			& \textbf{SDIFT}
			& \textbf{CATTE}
			& \textbf{OFTD}
			& \textbf{LRTFR}
			& \textbf{Senseiver} \\
			\midrule
			
			Gaussian noise ($\sigma=0.1$)
			& $\mathbf{0.077 \pm 0.007}$
			& $0.221 \pm 0.035$
			& $0.160 \pm 0.033$
			& $0.583 \pm 0.199$
			& $0.305 \pm 0.068$
			& $0.261 \pm 0.050$ \\
			
			Gaussian noise ($\sigma=0.3$)
			& $\mathbf{0.108 \pm 0.009}$
			& $0.230 \pm 0.046$
			& $0.167 \pm 0.038$
			& $0.615 \pm 0.180$
			& $0.306 \pm 0.067$
			& $0.298 \pm 0.054$ \\
			
			Laplacian noise ($\sigma=0.1$)
			& $\mathbf{0.083 \pm 0.009}$
			& $0.235 \pm 0.044$
			& $0.179 \pm 0.055$
			& $0.583 \pm 0.203$
			& $0.302 \pm 0.069$
			& $0.269 \pm 0.051$ \\
			
			Laplacian noise ($\sigma=0.3$)
			& $\mathbf{0.139 \pm 0.012}$
			& $0.263 \pm 0.042$
			& $0.196 \pm 0.045$
			& $0.616 \pm 0.187$
			& $0.304 \pm 0.069$
			& $0.326 \pm 0.065$ \\
			
			Poisson noise ($\sigma=0.1$)
			& $\mathbf{0.075 \pm 0.005}$
			& $0.228 \pm 0.042$
			& $0.163 \pm 0.062$
			& $0.598 \pm 0.213$
			& $0.307 \pm 0.068$
			& $0.333 \pm 0.085$ \\
			
			Poisson noise ($\sigma=0.3$)
			& $\mathbf{0.084 \pm 0.009}$
			& $0.241 \pm 0.049$
			& $0.181 \pm 0.044$
			& $0.624 \pm 0.212$
			& $0.308 \pm 0.067$
			& $0.421 \pm 0.111$ \\
			
			\bottomrule
		\end{tabular}
	}
\end{table}

\paragraph{Robustness to sensor drift and distribution shifts.}
Real-world sensors may exhibit temporally varying measurement biases due to device aging, calibration errors, or environmental changes.
To evaluate the robustness of StreamPhy against such distribution shifts, we consider two representative types of sensor drift: gradual drift and recurring drift.

For gradual drift, we introduce an exponential temporal drift into the measurements during inference:
\begin{equation}
	d_{\mathrm{grad}}(t)
	=
	\alpha
	\exp\left(
	\frac{t-T}{T}
	\right),
	\label{eq:gradual_drift}
\end{equation}
where $t$ denotes the current timestamp, $T$ denotes the total temporal horizon, and $\alpha$ controls the drift magnitude.

The results are reported in Table~\ref{tab:gradual_drift}.
StreamPhy maintains stable reconstruction performance under both mild and stronger gradual distribution shifts.

\begin{table}[t]
	\centering
	\caption{
		Reconstruction performance under gradual sensor drift on the Active Matter Dataset with uniform sampling ratio $\rho=1\%$.
	}
	\label{tab:gradual_drift}
	\small
	\resizebox{\textwidth}{!}{
		\begin{tabular}{lcccccc}
			\toprule
			\textbf{Drift Magnitude}
			& \textbf{StreamPhy}
			& \textbf{SDIFT}
			& \textbf{CATTE}
			& \textbf{OFTD}
			& \textbf{LRTFR}
			& \textbf{Senseiver} \\
			\midrule
			
			$\alpha=0.1$
			& $\mathbf{0.082 \pm 0.007}$
			& $0.262 \pm 0.032$
			& $0.201 \pm 0.025$
			& $0.557 \pm 0.231$
			& $0.317 \pm 0.062$
			& $0.276 \pm 0.047$ \\
			
			$\alpha=0.3$
			& $\mathbf{0.132 \pm 0.010}$
			& $0.316 \pm 0.033$
			& $0.303 \pm 0.034$
			& $0.591 \pm 0.216$
			& $0.386 \pm 0.045$
			& $0.383 \pm 0.071$ \\
			
			\bottomrule
		\end{tabular}
	}
\end{table}

For recurring drift, we introduce a periodic temporal perturbation following~\cite{gama2014survey}:
\begin{equation}
	d_{\mathrm{rec}}(t)
	=
	\alpha
	\sin\left(
	4\pi\frac{t}{T}
	\right),
	\label{eq:recurring_drift}
\end{equation}
where $\alpha$ controls the magnitude of the periodic drift.

As shown in Table~\ref{tab:recurring_drift}, StreamPhy remains robust under recurring temporal distribution shifts across different drift magnitudes.

\begin{table}[t]
	\centering
	\caption{
		Reconstruction performance under recurring sensor drift on the Active Matter Dataset with uniform sampling ratio $\rho=1\%$.
	}
	\label{tab:recurring_drift}
	\small
	\resizebox{\textwidth}{!}{
		\begin{tabular}{lcccccc}
			\toprule
			\textbf{Drift Magnitude}
			& \textbf{StreamPhy}
			& \textbf{SDIFT}
			& \textbf{CATTE}
			& \textbf{OFTD}
			& \textbf{LRTFR}
			& \textbf{Senseiver} \\
			\midrule
			
			$\alpha=0.1$
			& $\mathbf{0.081 \pm 0.011}$
			& $0.273 \pm 0.075$
			& $0.208 \pm 0.053$
			& $0.566 \pm 0.231$
			& $0.317 \pm 0.063$
			& $0.280 \pm 0.045$ \\
			
			$\alpha=0.3$
			& $\mathbf{0.147 \pm 0.027}$
			& $0.363 \pm 0.073$
			& $0.346 \pm 0.089$
			& $0.606 \pm 0.191$
			& $0.401 \pm 0.043$
			& $0.397 \pm 0.056$ \\
			
			\bottomrule
		\end{tabular}
	}
\end{table}

\paragraph{Robustness to missing frames.}
Real-world sensing systems may suffer from incomplete temporal observations due to packet loss, communication failures, or temporary sensor outages.
Following Observation Setting~2 in SDIFT~\cite{sdift2025}, we consider a setting in which observations are available only for a subset of frames.

Table~\ref{tab:missing_frames} reports the reconstruction results under uniform sampling ratios of $\rho=1\%$ and $\rho=3\%$.
StreamPhy maintains low reconstruction errors under both sampling ratios, demonstrating its robustness to incomplete temporal observation histories.

\begin{table}[t]
	\centering
	\caption{
		Reconstruction performance under missing frames on the Active Matter Dataset.
		Observations are available only for a subset of temporal frames.
		Results are reported as mean $\pm$ standard deviation.
	}
	\label{tab:missing_frames}
	\small
	\begin{tabular}{lcc}
		\toprule
		\textbf{Method}
		& \textbf{Uniform, $\rho=1\%$}
		& \textbf{Uniform, $\rho=3\%$} \\
		\midrule
		
		LRTFR
		& $0.620 \pm 0.484$
		& $0.598 \pm 0.527$ \\
		
		OFTD
		& $0.623 \pm 0.192$
		& $0.366 \pm 0.111$ \\
		
		CATTE
		& $0.319 \pm 0.091$
		& $0.271 \pm 0.046$ \\
		
		Senseiver
		& $0.272 \pm 0.048$
		& $0.207 \pm 0.058$ \\
		
		SDIFT
		& $0.296 \pm 0.096$
		& $0.256 \pm 0.087$ \\
		
		\textbf{StreamPhy}
		& $\mathbf{0.089 \pm 0.017}$
		& $\mathbf{0.081 \pm 0.018}$ \\
		
		\bottomrule
	\end{tabular}
\end{table}

\paragraph{Robustness to irregular temporal arrivals.}
We further evaluate the robustness of StreamPhy when observations arrive at variable temporal intervals.
Specifically, we simulate irregular temporal arrivals by selecting frames with different intervals, e.g., $\Delta t$, $2\Delta t$, and $3\Delta t$.
This setting represents practical scenarios involving asynchronous sensing or missing timestamps, where the temporal gaps between consecutive observations vary dynamically.

As shown in Table~\ref{tab:irregular_arrivals}, StreamPhy maintains stable reconstruction performance despite non-uniform temporal intervals.

\begin{table}[t]
	\centering
	\caption{
		Reconstruction performance under irregular temporal arrivals on the Active Matter Dataset.
		Observation intervals vary among $\Delta t$, $2\Delta t$, and $3\Delta t$.
		Results are reported as mean $\pm$ standard deviation.
	}
	\label{tab:irregular_arrivals}
	\small
	\begin{tabular}{lcc}
		\toprule
		\textbf{Method}
		& \textbf{Uniform, $\rho=1\%$}
		& \textbf{Uniform, $\rho=3\%$} \\
		\midrule
		
		LRTFR
		& $0.934 \pm 0.262$
		& $0.925 \pm 0.306$ \\
		
		OFTD
		& $0.689 \pm 0.244$
		& $0.672 \pm 0.187$ \\
		
		CATTE
		& $0.411 \pm 0.103$
		& $0.470 \pm 0.188$ \\
		
		Senseiver
		& $0.378 \pm 0.050$
		& $0.340 \pm 0.093$ \\
		
		SDIFT
		& $0.347 \pm 0.099$
		& $0.328 \pm 0.089$ \\
		
		\textbf{StreamPhy}
		& $\mathbf{0.125 \pm 0.042}$
		& $\mathbf{0.106 \pm 0.025}$ \\
		
		\bottomrule
	\end{tabular}
\end{table}

\paragraph{Robustness to delayed sensor measurements.}
We also consider temporal misalignment caused by delayed sensor measurements.
Specifically, we fix the sensor locations across all time steps and replace either $10\%$ or $30\%$ of the observations at time $t_i$ with measurements collected at the previous timestamp $t_{i-1}$.
This setting simulates delayed communication or synchronization errors that may arise in distributed sensing systems.

The results are summarized in Table~\ref{tab:delayed_sensors}.
StreamPhy maintains stable reconstruction performance even when a substantial fraction of the current measurements is replaced by temporally delayed observations.

\begin{table}[t]
	\centering
	\caption{
		Reconstruction performance under delayed sensor measurements on the Active Matter Dataset.
		A fraction of the observations at $t_i$ is replaced with measurements collected at $t_{i-1}$.
		Results are reported as mean $\pm$ standard deviation.
	}
	\label{tab:delayed_sensors}
	\small
	\resizebox{\textwidth}{!}{
		\begin{tabular}{lcccc}
			\toprule
			\textbf{Method}
			& \textbf{$\rho=1\%$, 10\% Delay}
			& \textbf{$\rho=3\%$, 10\% Delay}
			& \textbf{$\rho=1\%$, 30\% Delay}
			& \textbf{$\rho=3\%$, 30\% Delay} \\
			\midrule
			
			LRTFR
			& $0.341 \pm 0.097$
			& $0.266 \pm 0.089$
			& $0.372 \pm 0.113$
			& $0.291 \pm 0.105$ \\
			
			OFTD
			& $0.557 \pm 0.073$
			& $0.344 \pm 0.066$
			& $0.581 \pm 0.047$
			& $0.372 \pm 0.102$ \\
			
			CATTE
			& $0.170 \pm 0.033$
			& $0.166 \pm 0.035$
			& $0.183 \pm 0.047$
			& $0.172 \pm 0.029$ \\
			
			Senseiver
			& $0.2728 \pm 0.076$
			& $0.244 \pm 0.044$
			& $0.293 \pm 0.067$
			& $0.259 \pm 0.068$ \\
			
			SDIFT
			& $0.234 \pm 0.031$
			& $0.162 \pm 0.038$
			& $0.245 \pm 0.041$
			& $0.173 \pm 0.080$ \\
			
			\textbf{StreamPhy}
			& $\mathbf{0.0784 \pm 0.015}$
			& $\mathbf{0.0756 \pm 0.007}$
			& $\mathbf{0.081 \pm 0.011}$
			& $\mathbf{0.0791 \pm 0.009}$ \\
			
			\bottomrule
		\end{tabular}
	}
\end{table}

\paragraph{Computational efficiency and model complexity.}
In addition to robustness under practical sensing imperfections, we compare the computational efficiency and model complexity of StreamPhy with all baseline methods.
Table~\ref{tab:efficiency} reports the number of model parameters, training time, and inference time on the Turbulent Flow, Ocean Sound Speed, and Active Matter datasets.
For inference, we additionally report the runtime under different spatial sampling ratios to evaluate the effect of observation density on computational cost.

The parameter sizes vary across methods due to their different model architectures, and each baseline is individually configured for its corresponding experimental setting.
CATTE, OFTD, and LRTFR directly fit observations without a separate offline training stage; therefore, their training times are reported as N/A.
In contrast, StreamPhy performs offline training and subsequently reconstructs physical fields through efficient forward inference.

\begin{table*}[t]
	\centering
	\caption{
		Comparison of model size, training time, and inference time across different datasets.
		``N/A'' indicates that the corresponding method does not require a separate offline training stage.
	}
	\label{tab:efficiency}
	\small
	\resizebox{\textwidth}{!}{
		\begin{tabular}{llcccccc}
			\toprule
			\textbf{Dataset}
			& \textbf{Metric / Setting}
			& \textbf{StreamPhy}
			& \textbf{CATTE}
			& \textbf{OFTD}
			& \textbf{SDIFT}
			& \textbf{LRTFR}
			& \textbf{Senseiver} \\
			\midrule
			
			\multirow{4}{*}{Turbulent Flow}
			& Parameters
			& 10.3M
			& 6.5M
			& 5.1M
			& 15M
			& 4.2M
			& 8.3M \\
			
			& Training time
			& 152 min
			& N/A
			& N/A
			& 243 min
			& N/A
			& 178 min \\
			
			& Inference time ($\rho=3\%$)
			& \textbf{0.041 s}
			& 23 min
			& 16.18 s
			& 5.14 s
			& 45.1 s
			& 0.0553 s \\
			
			& Inference time ($\rho=5\%$)
			& \textbf{0.044 s}
			& 26 min
			& 17.59 s
			& 5.21 s
			& 53 s
			& 0.0787 s \\
			
			\midrule
			
			\multirow{4}{*}{Ocean Sound Speed}
			& Parameters
			& 9.94M
			& 5.4M
			& 4.9M
			& 15M
			& 3.2M
			& 8.5M \\
			
			& Training time
			& 154 min
			& N/A
			& N/A
			& 209 min
			& N/A
			& 211 min \\
			
			& Inference time ($\rho=1\%$)
			& \textbf{0.032 s}
			& 16 min
			& 6.09 s
			& 0.84 s
			& 27.5 s
			& 0.185 s \\
			
			& Inference time ($\rho=3\%$)
			& \textbf{0.043 s}
			& 16.5 min
			& 6.48 s
			& 0.89 s
			& 28.0 s
			& 0.196 s \\
			
			\midrule
			
			\multirow{4}{*}{Active Matter}
			& Parameters
			& 12M
			& 8.2M
			& --
			& 12M
			& 3.2M
			& 8.2M \\
			
			& Training time
			& 192 min
			& N/A
			& N/A
			& 278 min
			& N/A
			& 216 min \\
			
			& Inference time ($\rho=1\%$)
			& \textbf{0.047 s}
			& 18 min
			& 6.64 s
			& 1.31 s
			& 26.9 s
			& 0.16 s \\
			
			& Inference time ($\rho=3\%$)
			& \textbf{0.049 s}
			& 19 min
			& 6.74 s
			& 1.42 s
			& 27.3 s
			& 0.26 s \\
			
			\bottomrule
		\end{tabular}
	}
\end{table*}

As shown in Table~\ref{tab:efficiency}, StreamPhy consistently achieves sub-second inference across all three datasets and different observation densities.

\newpage

\section{Additional Visualizations}
\label{app:add_results}

\begin{figure}[h]
	\centering
	\includegraphics[width=\linewidth]{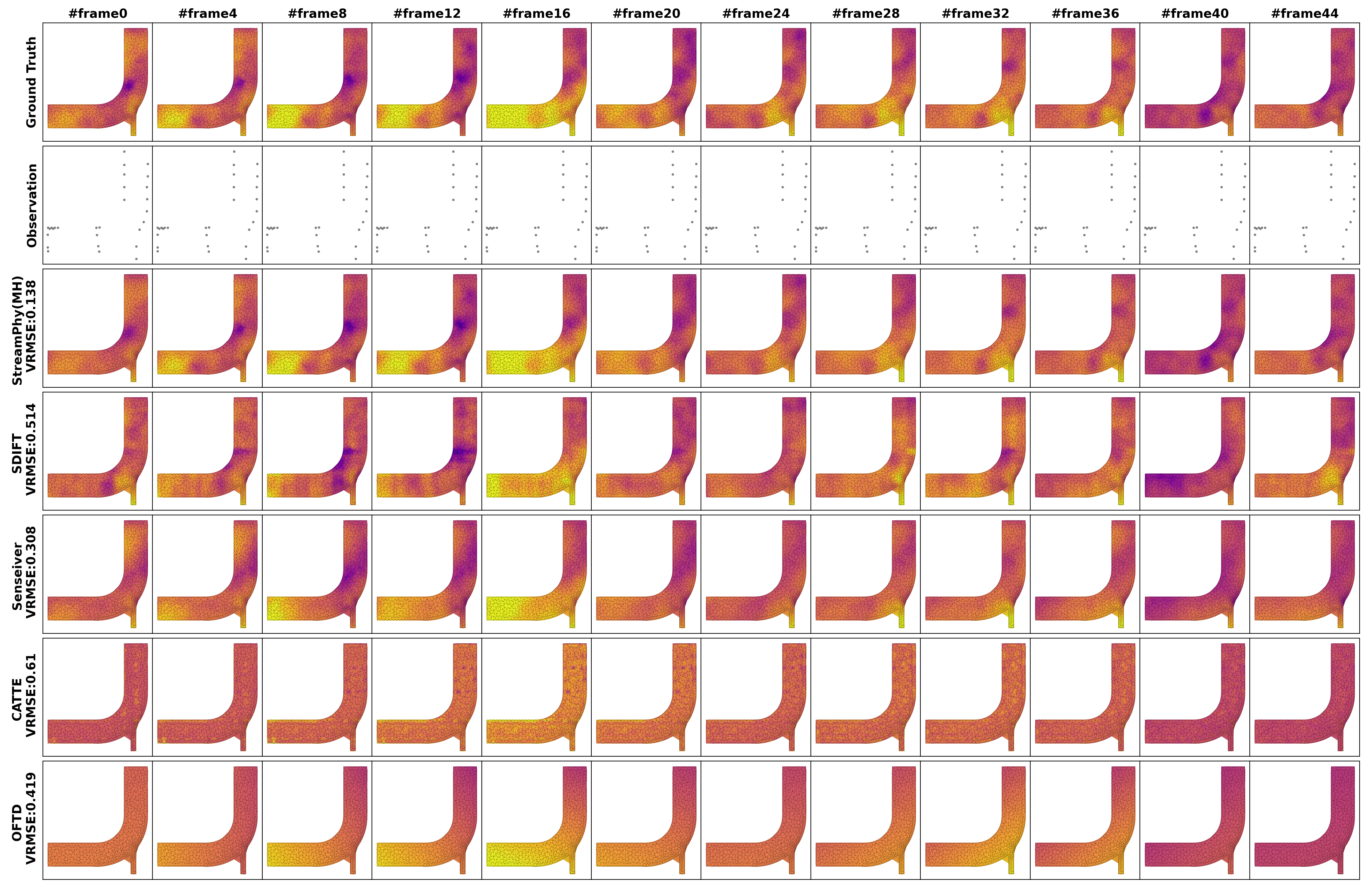}
	\caption{Reconstruction of \textit{Turbulent Flow} dynamics under slab sampling pattern with \(\rho = 3\%\).}
	\label{fig:tf_plot2}
	\vspace{0in}
\end{figure}
\begin{figure}[h]
	\centering
	\includegraphics[width=\linewidth]{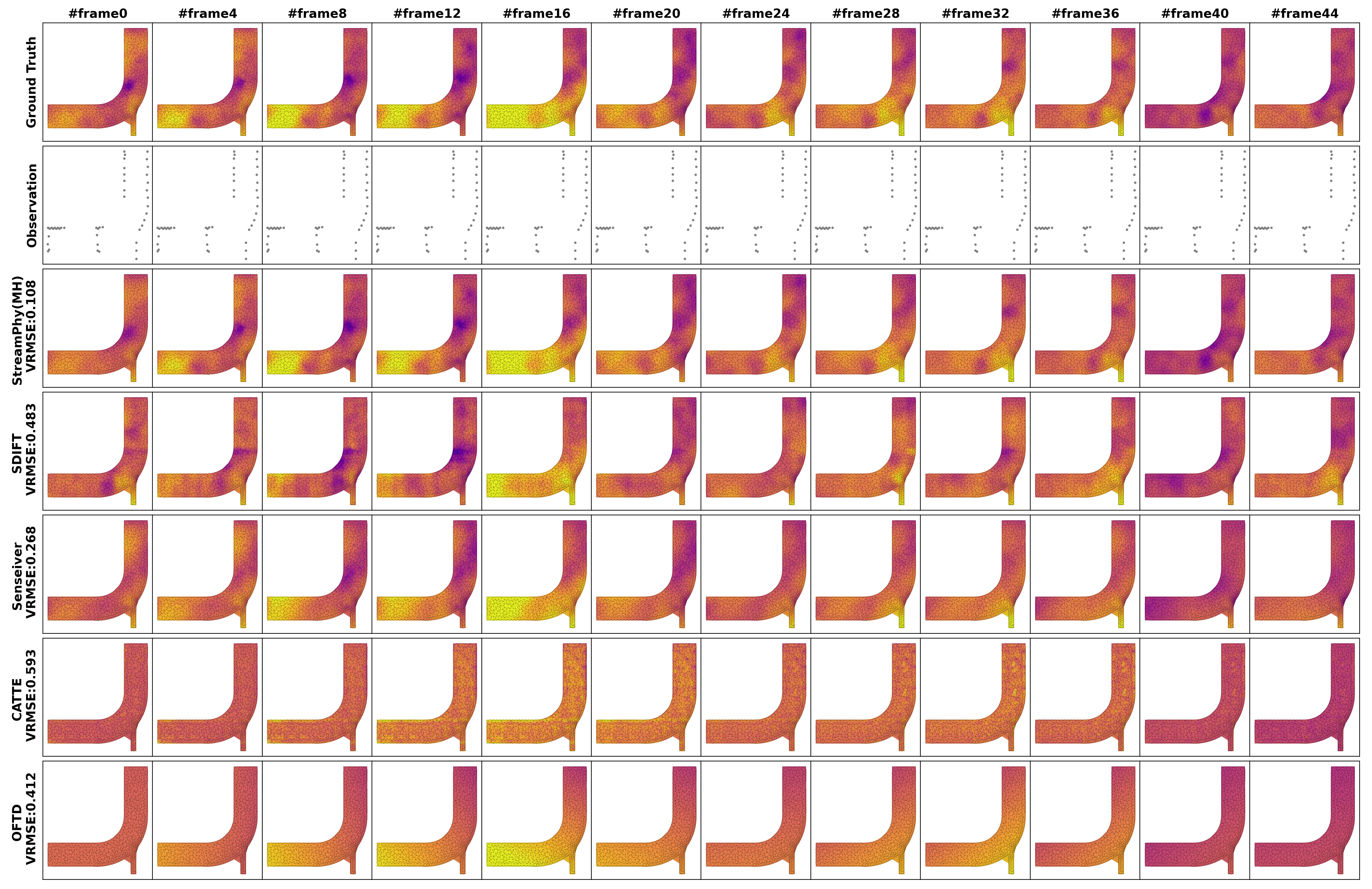}
	\caption{Reconstruction of \textit{Turbulent Flow} dynamics under slab sampling pattern with \(\rho = 5\%\).}
	\label{fig:tf_plot3}
	\vspace{0in}
\end{figure}
\begin{figure}[h]
	\centering
	\includegraphics[width=\linewidth]{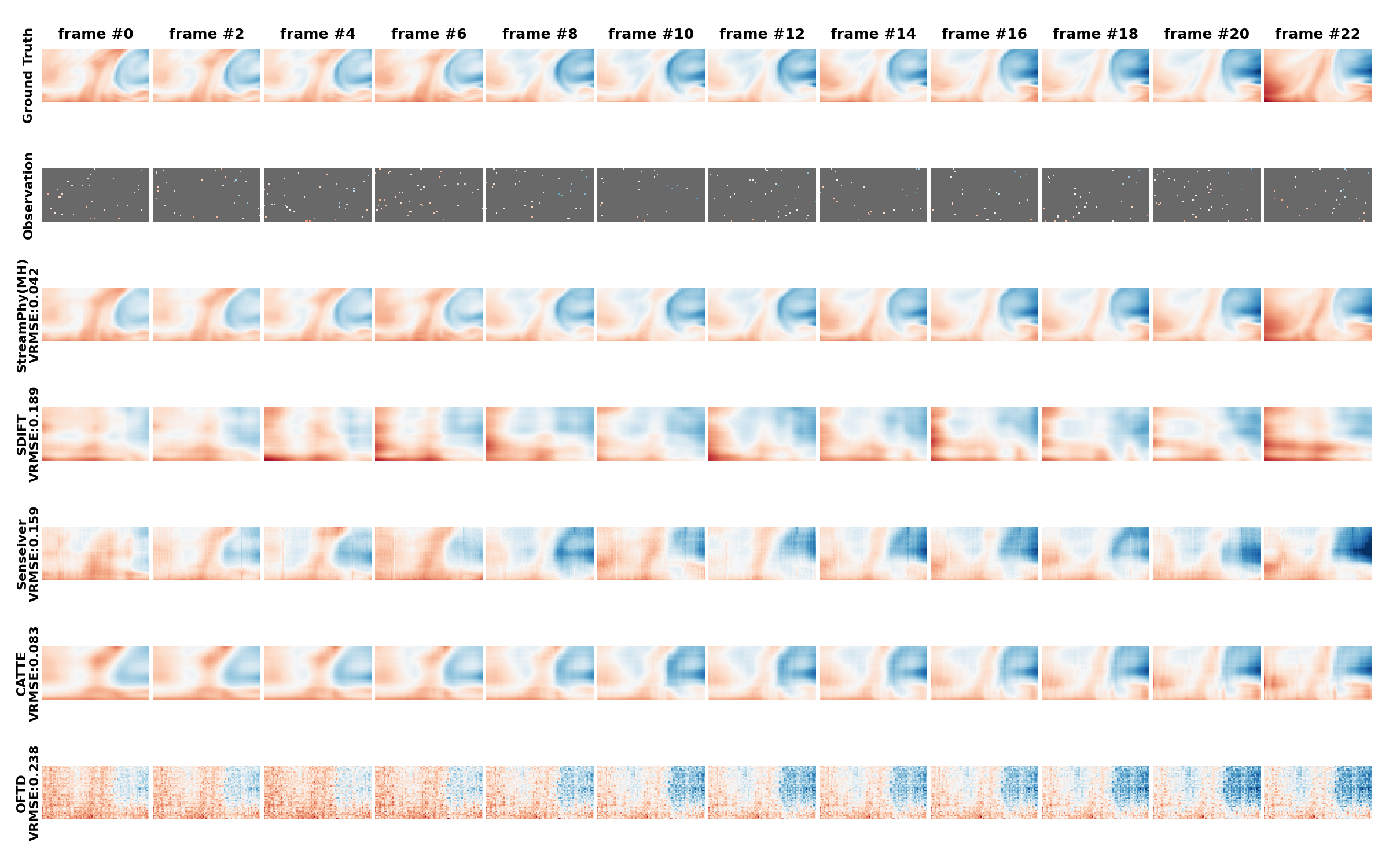}
	\caption{Reconstruction of \textit{Ocean Sound Speed} dynamics under uniform sampling pattern with \(\rho = 1\%\).}
	\label{fig:ssf_plot1}
	\vspace{0in}
\end{figure}
\begin{figure}[h]
	\centering
	\includegraphics[width=\linewidth]{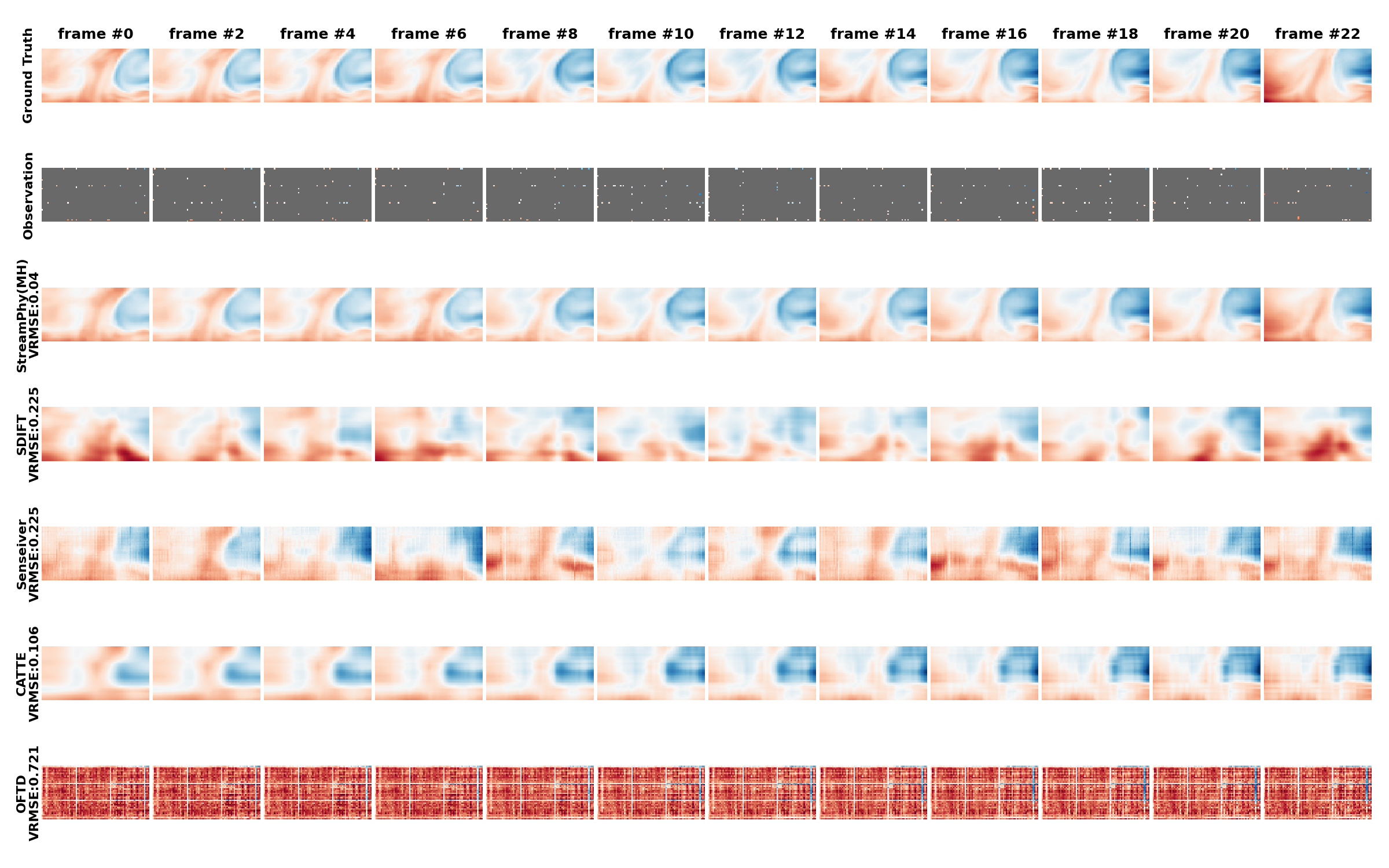}
	\caption{Reconstruction of \textit{Ocean Sound Speed} dynamics under slab sampling pattern with \(\rho = 1\%\).}
	\label{fig:ssf_plot2}
	\vspace{0in}
\end{figure}
\begin{figure}[h]
	\centering
	\includegraphics[width=\linewidth]{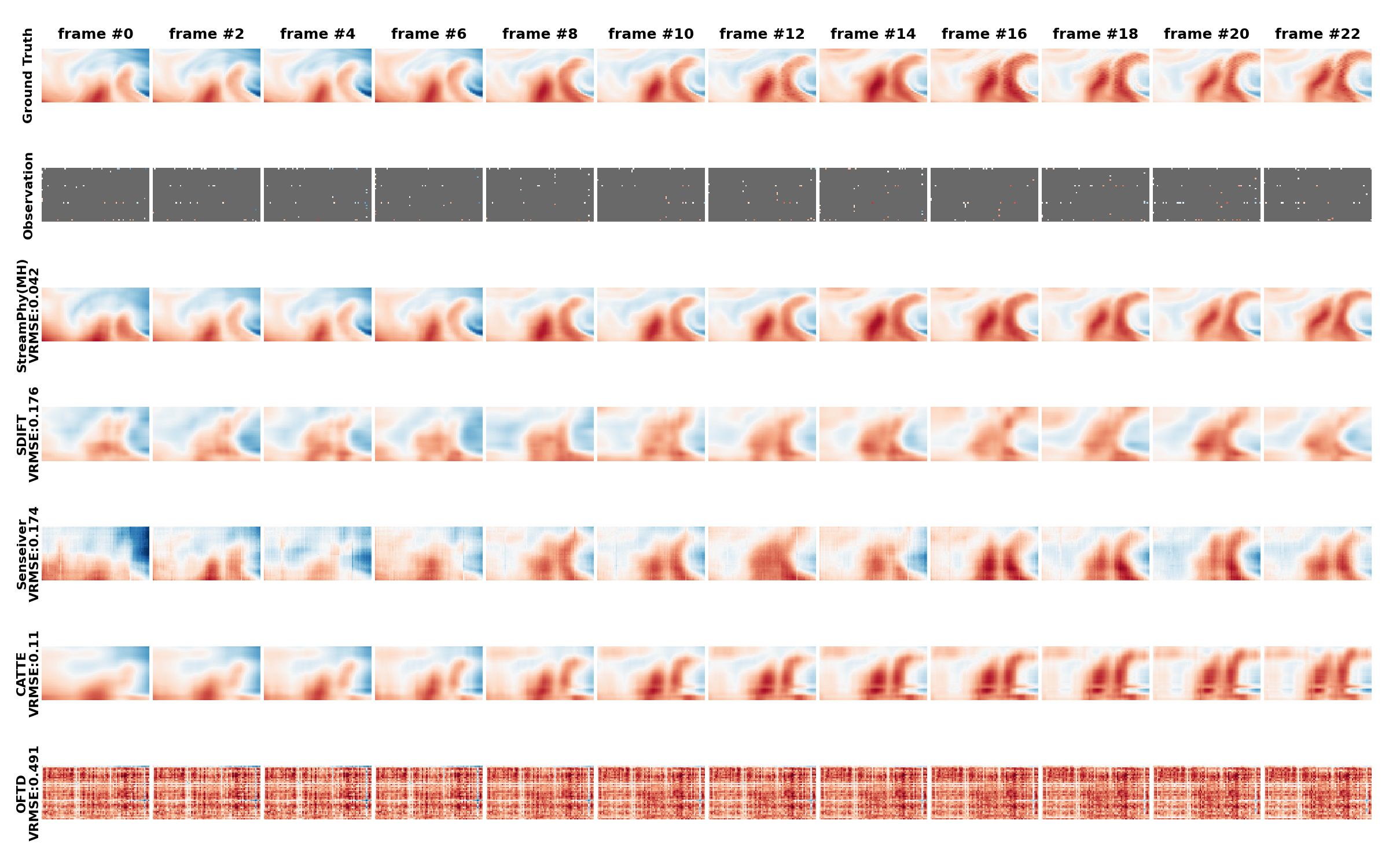}
	\caption{Reconstruction of \textit{Ocean Sound Speed} dynamics under slab sampling pattern with \(\rho = 1\%\).}
	\label{fig:ssf_plot3}
	\vspace{0in}
\end{figure}
\begin{figure}[h]
	\centering
	\includegraphics[width=\linewidth]{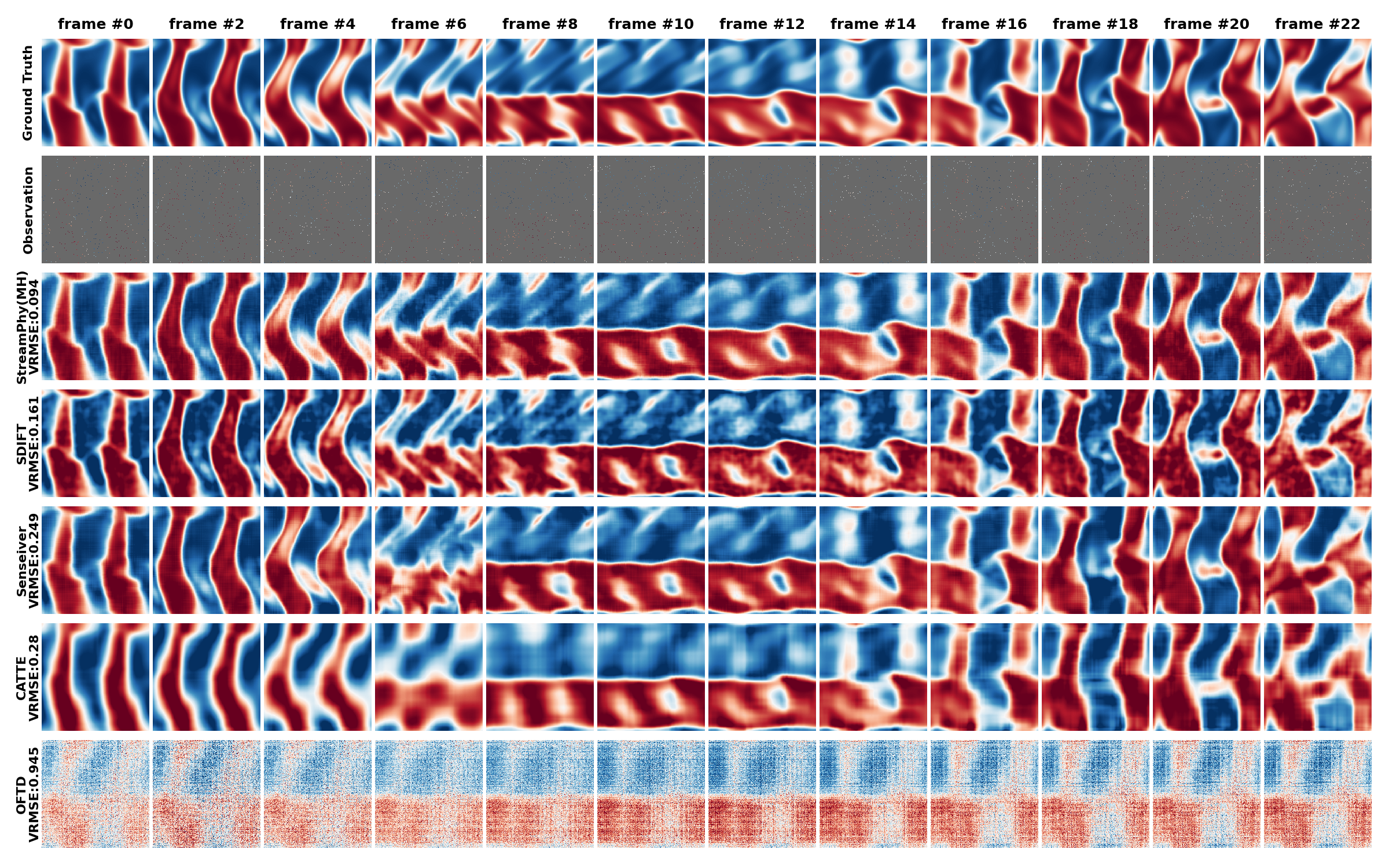}
	\caption{Reconstruction of \textit{Active Matter} dynamics under random sampling pattern with \(\rho = 1\%\).}
	\label{fig:ssf_plot4}
	\vspace{0in}
\end{figure}

\section{Limitations}
\label{app:limitations}
A current limitation of StreamPhy is that it does not explicitly enforce physical laws, such as strict boundary conditions or conservation constraints, during the modeling process. 
This is primarily a deliberate design choice motivated by our target real-world scenarios, where physical priors, boundary conditions, external forcing, or accurate governing equations may be incomplete, unavailable, or difficult to specify from sparse and heterogeneous observations. 
Therefore, the current framework focuses on learning continuous spatiotemporal representations directly from observational data without requiring a fully specified physical model. 
Nevertheless, the continuous spatial representation provided by the FT-FiLM decoder offers a natural interface for incorporating physics-based constraints when reliable physical priors are available. 
Since the decoder represents the physical field as a continuous function of spatial coordinates and time, its derivatives can be obtained via automatic differentiation, allowing PDE residuals, initial/boundary conditions, and conservation laws to be incorporated through PINN-style physics losses.

Another limitation is that the transition matrix $\mf{A}$ in the SSM is currently predefined, and its design space remains largely unexplored. 
In future work, we plan to integrate StreamPhy with domain-specific physical priors to improve the physical consistency and accuracy of long-range, large-scale physical field reconstruction. 
We also aim to investigate more flexible and potentially learnable parameterizations of $\mf{A}$ to further enhance the modeling capacity and performance of the SSM while preserving its ability to capture long-range dependencies.

\section{Impact Statement}
\label{app:b_i}
This paper focuses on advancing physical fields modeling techniques. We are  mindful of the broader ethical implications associated with technological progress in this field. Although immediate societal impacts may not be evident, we recognize the importance of maintaining ongoing vigilance regarding the ethical use of these advancements. It is crucial to continuously evaluate and address potential implications to ensure responsible development and application in diverse scenarios.

\clearpage

%%%%%%%%%%%%%%%%%%%%%%%%%%%%%%%%%%%%%%%%%%%%%%%%%%%%%%%%%%%%

%\newpage
%\input{checklist.tex}

\end{document}